\documentclass[10pt]{article}

\usepackage[accepted]{tmlr}

\usepackage[utf8]{inputenc}
\usepackage{adjustbox}
\usepackage{algorithm}
\usepackage{algpseudocode}
\usepackage{amsfonts}
\usepackage{amsmath}
\usepackage{amssymb}
\usepackage{amsthm}
\usepackage{bbm}
\usepackage{booktabs}
\usepackage{graphicx}
\usepackage{makecell}

\usepackage{mathtools}
\usepackage{float}
\usepackage{xcolor}
\usepackage{xspace}
\usepackage{subcaption}
\usepackage{tikz}
\usepackage{forest}
\usetikzlibrary{fit, backgrounds}
\usepackage{url}
\usepackage{hyperref}
\hypersetup{hidelinks}

\algrenewcommand{\algorithmiccomment}[1]{%
  \hfill\makebox[0.34\linewidth][l]{\footnotesize\textcolor{black!70}{$\triangleright$~#1}}%
}

\def\month{11}
\def\year{2026}
\def\openreview{\url{https://openreview.net/forum?id=RdQbKSjMHX}}

\newtheorem{proposition}{Proposition}
\newtheorem{theorem}{Theorem}
\newtheorem{assumption}{Assumption}
\newtheorem{corollary}{Corollary}

\newcommand{\cX}{\mathcal{X}}
\newcommand{\cY}{\mathcal{Y}}

\newcommand{\cD}{\mathcal{D}}
\newcommand{\cH}{\mathcal{H}}

\newcommand{\cR}{\mathcal{R}}

\newcommand{\R}{\mathbb{R}}

\newcommand{\E}[1]{\mathbb{E}\left[#1\right]}

\newcommand{\V}[1]{\textrm{Var}\left[#1\right]}

\renewcommand{\Pr}[1]{\mathbb{P}\left(#1\right)}

\newcommand{\Ind}[1]{\mathbb{I}\left[#1\right]}

\renewcommand{\cal}{\textrm{cal}}
\newcommand{\train}{\textrm{train}}

\newcommand{\ourmethod}{\texttt{LoBoost}\xspace}

\title{\ourmethod: Model-Native Local Conformal Prediction for Gradient-Boosted Trees}

\author{\name Vagner Silva Santos \\
\addr Department of Statistics, Federal University of São Carlos, São Carlos, São Paulo, Brazil\\
\addr Institute of Mathematics and Computer Science, University of São Paulo, São Carlos, São Paulo, Brazil
\AND
\name Victor Coscrato \\
\addr Department of Statistics, Federal University of São Carlos, São Carlos, São Paulo, Brazil
\AND
\name Luben M. C. Cabezas \\
\addr Department of Statistics, Federal University of São Carlos, São Carlos, São Paulo, Brazil\\
\addr Institute of Mathematics and Computer Science, University of São Paulo, São Carlos, São Paulo, Brazil\\
\addr Université Grenoble Alpes, Inria, CNRS, Grenoble INP, LJK, France
\AND
\name Rafael Izbicki \\
\addr Department of Statistics, Federal University of São Carlos, São Carlos, São Paulo, Brazil
\AND
\name Thiago R. Ramos \\
\addr Department of Statistics, Federal University of São Carlos, São Carlos, São Paulo, Brazil
}

\begin{document}

\maketitle

\begin{abstract}
Gradient-boosted decision trees are among the strongest off-the-shelf
predictors for tabular regression, but point predictions alone do not quantify
uncertainty. Conformal prediction provides distribution-free marginal
coverage, yet standard split conformal uses a single global residual quantile
and can adapt poorly to heteroscedasticity. We propose \ourmethod{}, a
model-native local conformal method that reuses the fitted ensemble's leaf
structure to define a multiscale partition of the feature space. Each input is
represented by the sequence of leaves it visits along the boosting path, and
matching leaf prefixes define nested groups in which residual quantiles are
estimated locally. By reusing the predictive structure already learned by the
model, \ourmethod{} requires no auxiliary partition or nuisance model, no
retraining, and no additional data split beyond standard conformal
calibration. Our theory connects the stability of the fitted ensemble and the
geometry of its induced cells to local residual-score homogeneity, providing
finite-sample coverage-error control and asymptotic pointwise validity as the
local calibration size grows and the cells become sufficiently homogeneous.
Experiments show a transparent quality--cost trade-off: \ourmethod{} improves over global split conformal on several heterogeneous datasets, but it is not uniformly superior to the adaptive baselines and is not uniformly Pareto-optimal.

\end{abstract}

\section{Introduction}

Gradient-boosted decision trees are among the strongest off-the-shelf predictors for tabular regression \citep{friedman2001greedy,chen2016xgboost,grinsztajn2022tree}.
In many applications, however, accurate point predictions are not enough: unobserved covariates, measurement noise, and intrinsic randomness make the response uncertain.
Practitioners therefore need  prediction intervals that adapt to each input and quantify uncertainty reliably.

Conformal prediction (CP) offers a simple and rigorous route to uncertainty quantification.
Given any predictor $g$ and an i.i.d.\ calibration set, split conformal constructs intervals $\widehat C(X)$ with finite-sample, distribution-free marginal coverage,
$\Pr{Y\in \widehat C(X)}\ge 1-\alpha$, under exchangeability \citep{vovk2005algorithmic,shafer2008tutorial,angelopoulos2023conformal,zhou2025conformal}.
In regression, the resulting interval width is controlled by an empirical quantile of calibration residuals.
Classical split conformal uses a single global residual quantile, which can yield overly wide intervals in low-noise regions and insufficient adaptivity under heteroscedasticity \citep{Lei2018,guan2023localized}.

A natural solution is \emph{partition-based} (or local) conformal prediction: estimate separate residual quantiles within regions $A\subseteq \cX$ and obtain coverage conditional on $X\in A$.
The practical challenge is choosing these regions: they must be small enough that residuals are approximately homogeneous, yet large enough to contain enough calibration points for stable quantile estimation.
Recent work addresses this trade-off by learning data-driven partitions with auxiliary trees or by defining soft neighborhoods in feature space \citep{bostrom2020mondrian,guan2023localized,martinez2024identifying,frohlich2025personalizedus,cabezas2025regression}.
A different line of work induces adaptivity by modifying the nonconformity score using additional nuisance estimates, such as conditional scales or quantiles \citep{Lei2018,romano2019conformalized,cabezasepistemic,izbicki2020flexible,izbicki2022cd,dheur2024distribution}.

While effective, local conformal methods introduce additional estimation beyond the base regressor. They typically learn an auxiliary partition or nuisance functions, or use extra splitting or cross-fitting, adding computational overhead beyond the fitted boosted model. For gradient-boosted trees, this overlooks a key fact: the trained ensemble already induces a rich, nested partition of $\cX$ via its leaf structure. We introduce \ourmethod, which reuses this model-induced partition to calibrate residual quantiles locally.
Concretely, we represent each input $x$ by the sequence of leaves it visits along the ensemble,
$
\boldsymbol{\ell}_T(x) = \bigl(\ell_1(x),\ldots,\ell_T(x)\bigr),
$ and define a local region around $x$ by requiring agreement on the first $k$ trees, i.e., matching the prefix $\boldsymbol{\ell}_k(x)=\bigl(\ell_1(x),\ldots,\ell_k(x)\bigr)$ for $k\leq T.$ This yields a nested, multiscale partition---coarse when $k$ is small, increasingly refined as $k$ grows---on which we calibrate residual quantiles.

The fitted ensemble therefore supplies the locality structure used by
\ourmethod{}. Rather than learning a new partition, the method searches along
this model-induced hierarchy for a resolution that preserves local structure
while retaining enough observations for stable calibration. Figure~\ref{fig:combined_figure} illustrates the motivation for \ourmethod{} on two 1D synthetic mechanisms with $n=2000$ points (split into $50\%$ training, $25\%$ calibration, and $25\%$ testing; details in Supplementary Material~\ref{sec:ap_sim}): \ourmethod{} produces locally adaptive intervals (narrow in low-noise regions, wide in high-noise regions) and attains conditional coverage closer to the nominal $1-\alpha=0.9$ than standard split conformal/ICP.

\begin{figure*}[!h]
    \centering
    \includegraphics[width=\linewidth]{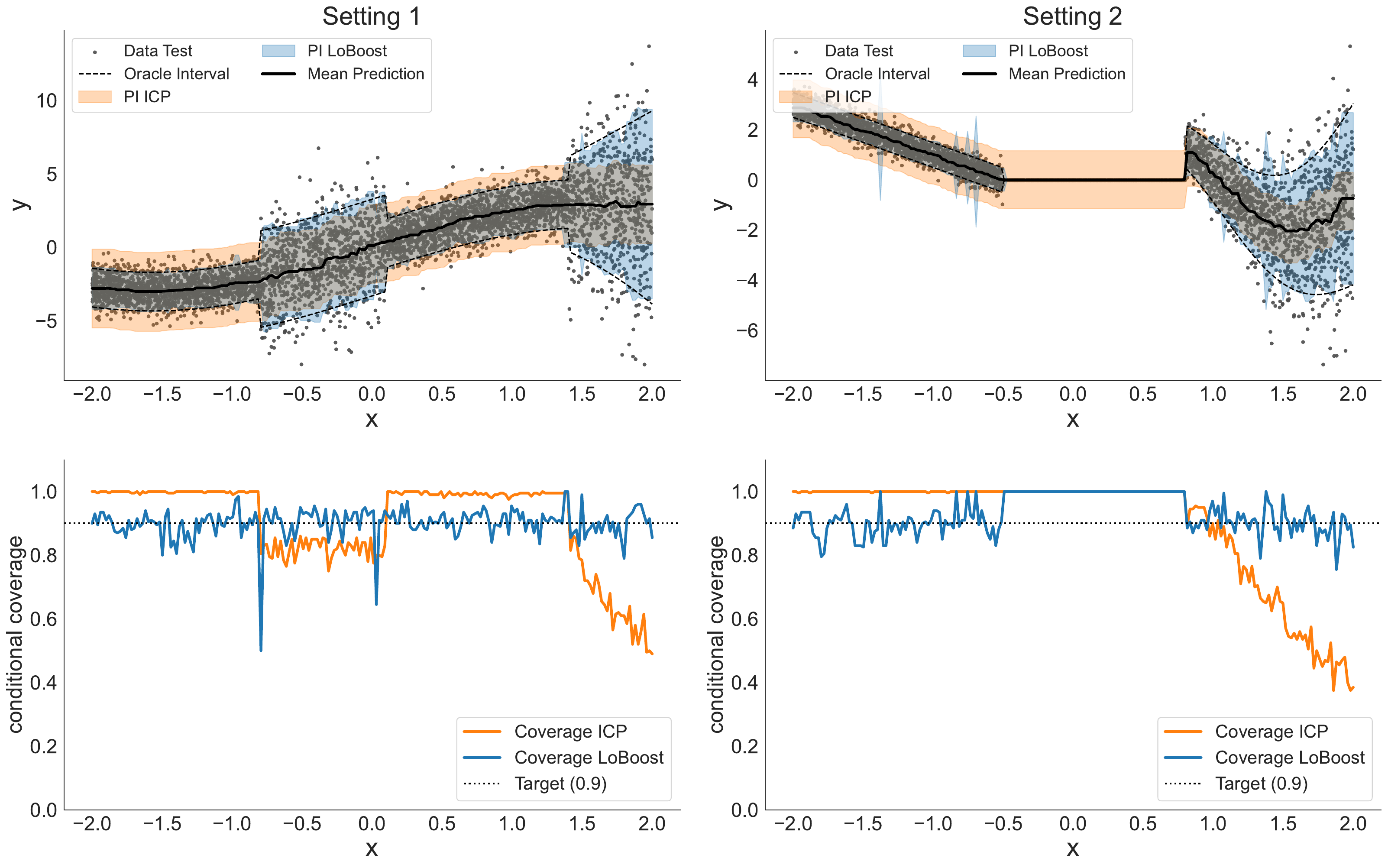}
    \caption{Comparison of prediction intervals and conditional coverage for \ourmethod and \textbf{ICP} on two synthetic mechanisms ($n=2000$, split $50/25/25$, base GBDT regressor with $100$ trees). Top: estimated intervals with the oracle interval and test data—\ourmethod tracks the oracle more closely, handling heteroscedasticity and discontinuities better than \textbf{ICP}. Bottom: empirical conditional coverage across the feature space (target $1-\alpha=0.9$)—\ourmethod stays near 0.9, while \textbf{ICP} deviates, reflecting \ourmethod{}'s meaningful feature-space partitions.}

    \label{fig:combined_figure}
\end{figure*}

\paragraph{Contributions.}

\begin{itemize}
\item \textbf{Model-native local conformal prediction for boosted trees.} We use leaf-index prefixes from the fitted boosted ensemble to define a multiscale partition of $X$ and perform conformal calibration within these cells. Relative to standard split conformal (ICP), which shares the same fitted predictor and calibration split, \ourmethod{} has significantly lower SMIS on eight of the eleven benchmarks and higher WSC on all eleven; however, it does not uniformly match the strongest adaptive baseline.

\textbf{Model-native, retraining-free calibration.} \ourmethod{} applies to an already fitted boosted model using only its leaf indices and a standard calibration set, with no auxiliary partition model, no nuisance estimator, and no retraining of the base learner. Its post-fit cost is lower than WICP and MICP on several larger benchmarks, but LoCART and ICP are generally faster; the revised paper therefore frames the method as a quality--cost trade-off rather than as uniformly fast.

\item \textbf{Theory connecting boosting stability to approximate conditional coverage.}
Using deterministic bounds based on the fitted trees' leaf values and
local response-regularity conditions, we show that residual-score
distributions become asymptotically homogeneous within shrinking
boosting-induced cells. These results provide finite-sample coverage-error
control and asymptotic pointwise validity as the model-induced cells become
local and their calibration sizes grow.
\end{itemize}

\subsection{Relation to Other Work}
\label{sec:relatedwork}
\paragraph{Uncertainty quantification for boosted models (non-conformal).}
Recent approaches for endowing boosting with uncertainty quantification generally fall into two categories.
First, probabilistic frameworks like NGBoost \citep{duan2020ngboost} and PGBM \citep{sprangers2021probabilistic}
assume a conditional parametric distribution, optimizing its parameters via natural gradients to minimize a proper scoring rule.
Second, statistical inference methods, such as the work by \cite{fangstatistical2025}, rely on Central Limit Theorems
to construct intervals based on the asymptotic normality of the boosted estimator.
However, both strategies lack the rigorous finite-sample, distribution-free validity of conformal prediction, relying instead
on specific parametric forms or asymptotic approximations that may not hold in practice.
Thus, these methods are model-native but rely on a parametric response
model or an asymptotic approximation.  \ourmethod{} is model-native and its
coverage analysis requires no parametric response model: it connects the
geometry and stability of the model-induced partition to nonparametric
finite-sample coverage-error control.

\paragraph{Conformal prediction with boosting as a base learner.}
Conformal prediction provides finite-sample, distribution-free coverage for any underlying regressor, including boosted trees.
Most existing work uses boosting primarily as a strong black-box predictor inside conformal wrappers (e.g., CQR-style efficiency gains, residual scoring, or distributional regression) \citep{romano2019conformalized,chernozhukov2021distributional,xu2024development,lai2025coffeeboost}.
In contrast, we leverage the \emph{internal} structure of the boosted
ensemble---its sequence of tree partitions---to define locality for
calibration. The fitted model itself supplies a nested family of local
regions, avoiding the need to learn an auxiliary partition estimator.
In a different approach,  \citet{xie2024boosted} propose Boosted CP Intervals, a post-hoc method that boosts a conformity score to improve interval length or conditional coverage. Unlike our approach, which uses leaf-based partitions of the boosted tree for local calibration, theirs learns a new score from model outputs without using the tree structure.

\paragraph{Local and partition-based conformal prediction.}
Exact conditional coverage is generally impossible without strong assumptions, motivating local or group-conditional targets \citep{lei2014distribution,vovk2012conditional,foygel2021limits,Lei2018,Izbicki2025}.
Local CP has been studied via kernel/neighborhood weighting \citep{guan2023localized}, multivalid/group-conditional optimization \citep{jung2023batch,kiyani2024conformal}, and Mondrian-style partitions \citep{vovk2005algorithmic,bostrom2020mondrian}.
Recently, tree-based variants were employed to learn residual-driven partitions using auxiliary trees trained on a calibration set \citep{martinez2024identifying,cabezas2025regression,cabezas2025conformalcalibrationstatisticalconfidence,cabezas2025cp4sbi}.
By contrast, we reuse the base model's internal tree-induced codes to define groups, requiring no auxiliary partition estimator.

\section{Background}

We consider a dataset of size $n$, denoted by
$\cD = \{(X_i, Y_i)\}_{i=1}^n,$
where the pairs $(X_i, Y_i)$ are independent and identically distributed according to a joint distribution $P$, with $(X_i, Y_i) \in \cX \times \cY$.
We assume that $\cD$ is partitioned into a training set $\cD_{\train}$ and a calibration set $\cD_{\cal}$, with respective sizes $|\cD_{\train}| = n_{\train}$ and $|\cD_{\cal}| = n_{\cal}$.

\subsection{Conformal Prediction}

Conformal prediction constructs prediction sets with finite-sample coverage under
exchangeability \citep{vovk2005algorithmic}. We focus on \emph{split conformal} \citep{Lei2018}: train a regressor $g$ on
$\cD_{\train}$ and, after reindexing the calibration observations, calibrate on $\cD_{\cal}=\{(X_j,Y_j)\}_{j=1}^{n_{\cal}}$ using a
nonconformity score $s(x,y)$. Define calibration scores $
S_j \coloneqq s(X_j,Y_j), \  j=1,\dots,n_{\cal}.
$ Let $S_{(1)}\le\cdots\le S_{(n_{\cal})}$ be the sorted scores, define
$r_{n_{\cal}}\coloneqq\lceil(n_{\cal}+1)(1-\alpha)\rceil$, and use the extended-real conformal quantile
\[
\widehat q_{1-\alpha}
\coloneqq
\begin{cases}
S_{(r_{n_{\cal}})}, & r_{n_{\cal}}\le n_{\cal},\\[3pt]
+\infty, & r_{n_{\cal}}=n_{\cal}+1.
\end{cases}
\]
The split-conformal
prediction set is
\[
\widehat C(x)\;\coloneqq\;\{y\in\cY:\ s(x,y)\le \widehat q_{1-\alpha}\}.
\]
Under exchangeability of the calibration data and  test pair $(X,Y)$,
\[
\Pr{Y\in \widehat C(X)}\;\ge\;1-\alpha.
\]

In regression, a commonly used score is the absolute-residual score $s(x,y)=|y-g(x)|$, yielding the interval
\[
\widehat C(x)=[\,g(x)-\widehat q_{1-\alpha},\ g(x)+\widehat q_{1-\alpha}\,].
\]
This is the score we adopt throughout this work.
% In regression, a commonly used score is the absolute-residual score $s(x,y)=|y-g(x)|$, yielding the
% interval
% $
% \widehat C(x)=[\,g(x)-\widehat q_{1-\alpha},\ g(x)+\widehat q_{1-\alpha}\,].
% $

\paragraph{Partition-based conformal.}
To obtain locally adaptive thresholds \citep{Lei2018}, let $\Pi=\{A_1,\dots,A_K\}$ be a partition of
$\cX$ (determined without using calibration responses). For each cell $A\in\Pi$, let
\[
I_A\coloneqq\{j:\ X_j\in A\} \text{ and } M_A\coloneqq |I_A|.
\]
Conditional on $M_A=m$, define $r_m\coloneqq\lceil(m+1)(1-\alpha)\rceil$ and
\[
\widehat q_A
\coloneqq
\begin{cases}
S^A_{(r_m)}, & r_m\le m,\\[3pt]
+\infty, & r_m=m+1,
\end{cases}
\]
where $S^A_{(r_m)}$ is the $r_m$-th order statistic of $\{S_j:j\in I_A\}$.
Then set $\widehat q(x)=\widehat q_{A(x)}$ for the unique $A(x)\in\Pi$ containing
$x$, and define
$\widehat C(x)\coloneqq\{y:\ s(x,y)\le \widehat q(x)\}.
$
Under exchangeability of calibration and test data, for every cell $A$,
\[\Pr{Y\in \widehat C(X)\mid X\in A} \geq 1-\alpha.\]
With $s(x,y)=|y-g(x)|$, this gives piecewise-constant local intervals
$[\,g(x)\pm \widehat q(x)\,]$.

\subsection{Gradient Boosting}\label{sec:gb}

Gradient boosting can be interpreted as (approximate) \emph{functional gradient descent} on a risk
functional \citep{friedman2001greedy,grubb2012generalizedboostingalgorithmsconvex_edge}. We present the
squared-loss view, which is enough to motivate the sequential, coarse-to-fine structure we exploit later.

Let $(X,Y)\sim P$, with $X\in\cX$ and $Y\in\R$. Consider the squared-risk functional
\[
\mathfrak L(g)\;\coloneqq\;\frac12\,\E{(Y-g(X))^2}.
\]
A standard calculation yields the functional gradient
\[
-\nabla \mathfrak L(g)(x)=\E{Y-g(X)\mid X=x},
\]
so the negative gradient corresponds to the (population) regression residual.

Boosting builds an additive predictor by iteratively fitting a weak learner to the current residuals.
Let $\cH$ denote the weak learner class, in our case, regression trees. Given training data
$\{(X_i,Y_i)\}_{i=1}^{n_{\train}}$ and a current predictor $g_{t-1}$, define residuals
$u_i^{(t)} \coloneqq Y_i-g_{t-1}(X_i)$ and fit
\[
h_t = \arg\min_{h\in\cH}\frac1{n_{\train}}\sum_{i=1}^{n_{\train}} \bigl(u_i^{(t)}-h(X_i)\bigr)^2.
\]
The model is then updated by the standard gradient-boosting additive step
\[
g_t(x)=g_{t-1}(x)+\eta\,h_t(x),
\qquad t=1,\dots,T,
\]
starting from an initial predictor $g_0$ (e.g., a constant). Unrolling the recursion yields the usual
additive form
\begin{equation}\label{eq:boosting_definition}
g_T(x)=g_0(x)+\sum_{t=1}^T \eta\,h_t(x),
\qquad h_t\in\cH,
\end{equation}
i.e., a sequence of progressively refined corrections.
 Because each new tree is trained on the \emph{remaining} residual signal, boosting behaves like a descent
procedure: as iterations proceed, the residuals typically shrink, so later trees act as smaller corrective
refinements (especially under shrinkage $\eta$).  This coarse-to-fine behavior is the key intuition
behind defining multiscale regions using \emph{prefixes} of boosting paths in
Section~\ref{sec:ourmethod}.

\section{Our Method: \ourmethod}\label{sec:ourmethod}

Next, we describe how the internal structure of a gradient-boosted ensemble can be leveraged to construct partition-based conformal predictors.

\subsection{Boosting-Induced Partitions}\label{sec:boost_partitions}

The sequential nature of gradient boosting induces a hierarchical organization of the feature space: early trees capture coarse, global structure, while later trees introduce progressively finer refinements.
As a result, the boosting procedure naturally defines a multiscale, nested decomposition of the feature space. This viewpoint is inspired by earlier work on task-specific similarity learning and boosting-based blocking schemes \citep{shakhnarovich2005learning,ramos24a}.

Let $\ell_t(x)$ denote the leaf identifier reached by a point $x\in\cX$ in the $t$-th regression tree $h_t$, as defined in~\eqref{eq:boosting_definition}.
The boosted ensemble therefore associates each input with a sequence of leaf identifiers,
\[
\boldsymbol{\ell}_T(x)
    \;=\;
    \bigl(\ell_1(x),\,\ell_2(x),\,\ldots,\,\ell_T(x)\bigr),
\]
which provides a hierarchical encoding of the feature space induced by the boosting trajectory.
Specifically, for a resolution level $k\le T$, we say that two points $x$ and $x'$ belong to the same region at depth $k$ if they agree on the first $k$ entries of this encoding:
\[
x \sim_k x'
\quad\Longleftrightarrow\quad
\ell_t(x) = \ell_t(x') \;\;\text{for all } t \le k.
\]
We denote by
\[
\cR_k(x)
    \;=\;
    \{x'\in \cX : x' \sim_k x\},
\]
the equivalence class of $x$ at depth $k$, which we refer to as a \emph{boosting-induced region}.

The common-prefix construction is useful for local calibration
because it yields an exact cancellation along the boosting path. If
$z\in\cR_k(x)$, then $\ell_t(z)=\ell_t(x)$ for every $t\le k$. Since each regression
tree is constant within a leaf,
\[
h_t(z)=h_t(x),
\qquad t\le k.
\]
Consequently,
\begin{align*}
g_T(z)-g_T(x)
=
\eta\sum_{t=1}^T\bigl(h_t(z)-h_t(x)\bigr)=
\eta\sum_{t=k+1}^T\bigl(h_t(z)-h_t(x)\bigr).
\end{align*}
Thus, the first $k$ fitted components cancel exactly, and two points in the
same boosting-induced region can differ only through the subsequent
corrections. In boosting regimes where later trees make progressively smaller
corrections, this remaining variation is expected to be small. 

To quantify  the remaining variation at both pointwise and uniform levels we define the leaf-value distance
\[
d_{\mathrm{LV}}(x,z)
\coloneqq
\eta\sum_{t=1}^T |h_t(x)-h_t(z)|,
\]
the output range of tree $t$,
\[
B_t\coloneqq\sup_{u,z\in\cX}|h_t(u)-h_t(z)|,
\]
and the tail envelope
\[
b_{T,k}\coloneqq\eta\sum_{t=k+1}^T B_t.
\]
The distance $d_{\mathrm{LV}}(x,z)$ is a pointwise envelope on the prediction
variation accumulated along two boosting paths. The quantity $B_t$ is the
largest contribution that tree $t$ can make to the difference between two
predictions, while $b_{T,k}$ uniformly bounds the variation that remains after
the first $k$ trees have canceled. All three quantities are determined by the
realized fitted ensemble and are finite whenever its leaf values are finite.

\begin{proposition}[Leaf-value stability]
\label{proposition:leaf_value_stability}
For every $x,z\in\cX$,
\[
|g_T(x)-g_T(z)|\le d_{\mathrm{LV}}(x,z).
\]
Consequently, for the absolute-residual score $s(x,y)=|y-g_T(x)|$,
\[
|s(x,y)-s(z,y)|\le d_{\mathrm{LV}}(x,z)
\qquad\text{for every }y\in\R.
\]
Moreover, if $z\in\cR_k(x)$, then
\[
d_{\mathrm{LV}}(x,z)\le b_{T,k}.
\]
\end{proposition}

Proposition~\ref{proposition:leaf_value_stability} shows that when $z\in\cR_k(x)$ and the tail envelope $b_{T,k}$ is small, the nonconformity score $s(z,y) = |y-g_T(z)|$ is approximately equal to $s(x,y)$; that is, the absolute-residual score remains approximately constant across the boosting-induced partitions.  Due to this property, this is the score we consider throughout this work.

% Hence, when the boosting tail is small, points in the same prefix cell have
% uniformly similar fitted centers and residual-score functions.

% Proposition~\ref{proposition:leaf_value_stability} makes the
% cancellation argument quantitative. For every $z\in\cR_k(x)$,
% \begin{align*}
% |g_T(z)-g_T(x)|
% &\le d_{\mathrm{LV}}(x,z)=\eta\sum_{t=k+1}^T|h_t(z)-h_t(x)|\le b_{T,k}.
% \end{align*}
% The same stability transfers directly to the absolute-residual score because
% it is Lipschitz in the fitted center: for every response value $y$,
% \[
% \bigl||y-g_T(z)|-|y-g_T(x)|\bigr|
% \le |g_T(z)-g_T(x)|
% \le b_{T,k}.
% \]

% Together with
% local regularity of the conditional response law and geometric localization
% of the cell, this score stability yields the within-cell score-distribution
% homogeneity required for local conformal calibration.

\subsection{Local Conformal Prediction with Boosted Partitions}

% Throughout the theoretical development, $\Pi$ denotes a partition,
% $A\in\Pi$ denotes one of its cells, $I_A$ and $M_A=|I_A|$ denote the
% within-cell calibration index set and its size, and $\widehat q_A$ denotes the
% corresponding local conformal cutoff.

We now use the cells supplied by the fitted boosting ensemble as
conformal calibration groups. The model induces an entire hierarchy of such
cells; to first isolate the calibration argument, consider a final partition
determined together with the fitted ensemble. Section~\ref{sec:practical} then
develops the choice of resolution within this hierarchy.

We let \(\mathcal F_{\mathrm{tr}} := \sigma(\cD_{\train})\) denote the
\(\sigma\)-field generated by the training data and we assume that both the
fitted predictor \(g_T\) and the entire final partition are
\(\mathcal F_{\mathrm{tr}}\)-measurable.

For $A\in\Pi$, recall that
\[
I_A\coloneqq\{i\le n_{\cal}:X_i\in A\},
\qquad
M_A\coloneqq|I_A|,
\]
and let $S_i=|Y_i-g_T(X_i)|$.  Conditional on $M_A=m$, set
\[
r_m\coloneqq\left\lceil(m+1)(1-\alpha)\right\rceil
\]
and define the local conformal cutoff by
\[
\widehat q_A
=
\begin{cases}
S^A_{(r_m)}, & r_m\le m,\\[3pt]
+\infty, & r_m=m+1,
\end{cases}
\]
where $S^A_{(r_m)}$ is the $r_m$-th order statistic of
$\{S_i:i\in I_A\}$.  The prediction interval is
\[
\widehat C(x)
=
[g_T(x)-\widehat q_{A(x)},\;g_T(x)+\widehat q_{A(x)}].
\]

\begin{theorem}[Finite-sample group-conditional validity]
\label{theorem:local_coverage}
Suppose that, conditionally on $\mathcal F_{\mathrm{tr}}$, the calibration
pairs and an independent test pair $(X_{n+1},Y_{n+1})$ are identically
distributed and exchangeable.  Then, for every cell $A\in\Pi$ with positive
probability,
\[
\Pr{
Y_{n+1}\in\widehat C(X_{n+1})
\,\middle|\,
\mathcal F_{\mathrm{tr}},\ X_{n+1}\in A
}
\ge 1-\alpha.
\]
Averaging over the cells also gives
\[
\Pr{
Y_{n+1}\in\widehat C(X_{n+1})
\,\middle|\,\mathcal F_{\mathrm{tr}}
}
\ge1-\alpha.
\]
\end{theorem}

\subsection{Asymptotic Conditional Coverage}
\label{sec:asymptotic_conditional_coverage}

Theorem~\ref{theorem:local_coverage} gives exact coverage conditionally on a
cell $A$, but not conditionally on a point $X=x$.  Without additional
structure, exact distribution-free conditional coverage at every point is
impossible
\citep{vovk2012conditional,foygel2021limits,Lei2018}.  We next show that the
cell-level guarantee becomes pointwise asymptotically when the cell contains
increasingly many calibration observations and its score distribution becomes
homogeneous around $x$.

Consider a sequence of fitted predictors $g_n$ and corresponding
final model-induced partitions $\Pi_n$, each measurable with respect to the
training sigma-field $\mathcal F_{\mathrm{tr},n}$, together with independent
calibration samples $\cD_{\cal,n}$ of size $n_{\cal,n}$. Let $A_n(x)$ be the cell containing
$x$, and let $\widehat C_n$ be obtained by applying the within-cell calibration
rule above. Define
\[
p_n(x)
\coloneqq
\Pr{X\in A_n(x)\mid\mathcal F_{\mathrm{tr},n}}
\]
and measure local score heterogeneity by
\[
\Delta_n(x)
\coloneqq
\sup_{t\ge0}
\left|
\Pr{|Y-g_n(X)|\le t
\mid\mathcal F_{\mathrm{tr},n},X\in A_n(x)}
-
\Pr{|Y-g_n(x)|\le t
\mid\mathcal F_{\mathrm{tr},n},X=x}
\right|.
\]

\begin{theorem}[Asymptotic conditional coverage]
\label{theorem:conditional_coverage}
Fix $\alpha\in(0,1)$ and $x\in\cX$.  Suppose that the conditional score
distribution in $A_n(x)$ is continuous and that, as $n\to\infty$,
\[
n_{\cal,n}p_n(x)\xrightarrow{p}\infty,
\qquad
\Delta_n(x)\xrightarrow{p}0.
\]
Then the realized pointwise coverage satisfies
\[
\Pr{Y\in\widehat C_n(x)
\mid X=x,\mathcal F_{\mathrm{tr},n},\cD_{\cal,n}}
\xrightarrow{p}
1-\alpha.
\]
\end{theorem}

The quantity $n_{\text{cal},n}p_n(x)$ is the conditional expected number of calibration observations in $A_n(x)$. Its divergence still allows the cell probability $p_n(x)$ to approach zero, and hence the partition to become more local, but rules out cells that shrink so quickly that their empirical quantiles remain noisy. This is a fairly reasonable assumption, since it simply requires that, on average, enough calibration points fall within each cell for the local quantile estimates to be reliable.

% The quantity $n_{\cal,n}p_n(x)$ is the conditional expected number of
% calibration observations in $A_n(x)$.  Its divergence still allows the cell
% probability $p_n(x)$ to approach zero, and hence the partition to become more
% local, but rules out cells that shrink so quickly that their empirical
% quantiles remain noisy.  In fact, the proof shows that this condition makes
% the realized number of calibration observations in the cell diverge.

% Increasing
% $k_n(x)$ makes the prefix cell more local and cannot increase $b_n(x)$, but it
% also cannot increase its probability $p_n(x)$.  Hence the sequence $k_n(x)$
% must be chosen so that the locality and tail errors vanish while
% $n_{\cal,n}p_n(x)\to\infty$.

Note that the sequence of cells in Theorem~\ref{theorem:conditional_coverage} is generated by a sequence of prefix depths
$
A_n(x)=\cR_{n,k_n(x)}(x).
$
Thus both $p_n(x)$ and $\Delta_n(x)$ in the theorem depend on the choice of the resolution level $k_n(x)$. Appendix~\ref{appendix:boosting_sufficient_conditions} shows that, under local response regularity, shrinking cell radius and a vanishing boosting tail $b_n(x)$ are sufficient for $\Delta_n(x)\to0$. How to choose $k_n(x)$ in practice is the subject of the next section.

\section{Resolution Selection and Aggregation}
\label{sec:practical}

There are two natural ways to select the resolution $k$. A conceptually
simple approach is to reserve an auxiliary holdout sample, disjoint from
the training and calibration samples, for choosing the partition. Once
the resolution is selected, the partition is fixed and the calibration
sample is used only to estimate the local quantiles. This decouples
resolution selection from calibration and yields validity under weaker
conditions.

In practice, we use a more data-efficient variant that does not require
this additional split. Namely, the calibration sample is used both to
select the resolution and to estimate the local quantiles. Consequently,
the exact finite-sample group-conditional guarantee of
Theorem~\ref{theorem:local_coverage}, which requires a partition fixed
before calibration, does not directly apply to this practical variant.
Instead, its validity is characterized by the finite-sample high-probability
bound in Theorem~\ref{theorem:calibration_covariate_adaptive}; the bound
depends on the realized local count and the within-cell score discrepancy.
This reuse of calibration covariates introduces an additional layer of
adaptivity and therefore requires stronger conditions in our theoretical
analysis. We describe
this practical selection rule next.

At resolution level $k$, the partition cells are indexed by the
prefix $\boldsymbol{\ell}_k(x)=\bigl(\ell_1(x),\ldots,\ell_k(x)\bigr)$, so their number can grow rapidly
with $k$. If the trees have depth at most $d$, the worst-case number of cells
is $O(2^{kd})$; even for decision stumps, it can be $2^k$. \ourmethod{}
therefore increases the resolution only within regions containing enough
calibration observations to support a stable local quantile.

\begin{figure}[H]
    \centering
    \resizebox{0.5\linewidth}{!}{%
        \begin{forest}
          for tree={
            draw,
            rounded corners,
            align=center,
            font=\sffamily\small,
            edge={->, >=stealth, thick, draw=gray},
            l sep=0.8cm,
            s sep=0.2cm,
          },
          root/.style={fill=blue!10, draw=blue},
          terminal/.style={fill=red!10, draw=red}
          [$\cD_{\cal}$, root
            [{$\boldsymbol{\ell}_1(x)=(0)$}
              [{$\boldsymbol{\ell}_2(x)=(0,0)$}, terminal]
              [{$\boldsymbol{\ell}_2(x)=(0,1)$}
                [{$\boldsymbol{\ell}_3(x)=(0,1,0)$}, terminal]
                [{$\boldsymbol{\ell}_3(x)=(0,1,1)$}, terminal]
              ]
            ]
            [{$\boldsymbol{\ell}_1(x)=(1)$}
              [{$\boldsymbol{\ell}_2(x)=(1,0)$}, terminal]
              [{$\boldsymbol{\ell}_2(x)=(1,1)$}, terminal]
            ]
          ]
        \end{forest}
    }
    \caption{\textbf{LoBoost partition}. For convenience, we
    consider $h_i(x)$ as decision stumps. White nodes indicate dense regions
    that satisfy the chosen local-size requirement and are therefore eligible
    for further refinement. Red nodes denote terminal regions that do not
    satisfy this requirement before aggregation. The example results in five
    terminal regions.}
    \label{fig:partition}
\end{figure}

More precisely, we extract the leaf paths
$\{\boldsymbol{\ell}_T(X_i)\}_{i=1}^{n_{\cal}}$ of the calibration
covariates. At level $k$, each active group is refined according to the
leaf identifiers of tree $k$. To ensure reliable local quantile estimation,
we define the effective local-size threshold
\begin{equation}\label{eq:nmin_def}
N_{\min}
\coloneqq
\max\Bigl\{N_{\mathrm{part}},
\lceil p_{\min}n_{\cal}\rceil\Bigr\},
\end{equation}
which combines an absolute floor $N_{\mathrm{part}}$ with a proportional
threshold $p_{\min}$. Groups with fewer than $N_{\min}$ calibration
observations become terminal, while the remaining groups proceed to the
next tree. Consequently, refinement can produce undersized terminal regions, which are consolidated during the aggregation step. Figure~\ref{fig:partition} illustrates this construction.

We aggregate raw terminal regions with fewer than $N_{\min}$
observations using the similarity encoded by their boosting paths. Before
refinement, the complete length-$T$ leaf signature
$\boldsymbol{\ell}_T(X_i)$ is extracted for every calibration observation.
Each raw terminal region $R$ is assigned a deterministic representative
observation $x_R\in R$ and retains the complete representative signature
$\boldsymbol{\ell}_T(R)\coloneqq\boldsymbol{\ell}_T(x_R)$. Thus, if $R$
becomes terminal at depth $k<T$, termination stops further refinement but
does not leave the later entries undefined: $\ell_t(R)$ for $t>k$ is read
from $\boldsymbol{\ell}_T(R)$. For two regions $R$ and $R'$, define the
weighted Hamming distance
\begin{align}
\label{eq:weightedhamming}
d_{\mathrm{WH}}(R,R')
&=
\sum_{t=1}^T
w_t\,\Ind{\ell_t(R)\neq\ell_t(R')},
&
w_t
&\coloneqq \eta B_t.
\end{align}

The range $B_t$ measures the maximum prediction variation produced by
tree $t$, so this weighting assigns a larger penalty to disagreements
in trees with greater influence on the fitted predictor. Starting from
the raw terminal regions, we repeatedly select the deepest undersized
region, breaking depth ties in favor of the smaller region, and merge it into the closest remaining region under $d_{\mathrm{WH}}$, pooling their calibration observations into a single group. The destination may itself be undersized. After merging $R$ into $R^\star$, the resulting region retains the complete representative signature $\boldsymbol{\ell}_T(R^\star)$ of the destination for subsequent distance computations. This process continues until no undersized region remains. For regions terminating at different depths, positions beyond the end of the shorter path are treated as disagreements.

Thus, aggregation does not modify the refinement tree. Instead, the aggregation step produces a many-to-one map from the raw (pre-merge) terminal regions into the aggregated groups.

After aggregation, a new observation $x$ is routed through the original refinement tree using its leaf path $\boldsymbol{\ell}_T(x)$, until it reaches a raw terminal region. It is then assigned to the final aggregated group containing that raw region and receives the corresponding local cutoff. If the observed leaf path does not lead to any raw terminal region represented during calibration, we fall back to the global calibration cutoff $\widehat q_{1-\alpha}$.

All refinement and aggregation decisions use only the fitted model and the
calibration covariates, not their responses; calibration responses are used only afterward to estimate the within-cell quantiles. 

The complete partitioning, aggregation, and calibration procedure is summarized in Algorithm~\ref{alg:tree_partitioning}.

Our implementation is open-source and available at \url{https://github.com/thiagorr162/loboost.git}. A key advantage of \ourmethod is computational reuse: after fitting the boosting
predictor, partition construction and calibration reuse its learned tree structure, with no retraining of the base model. Our implementation has native support for widely used Python boosting libraries, including scikit-learn's gradient boosting \citep{scikit-learn}, XGBoost \citep{chen2016xgboost}, and CatBoost \citep{prokhorenkova2019catboostunbiasedboostingcategorical}.

All the steps of our algorithm are computationally cheap:
\begin{itemize}
    \item Leaf identifiers extraction is $O(m \times T)$ at the worst case. We note that our algorithm adaptively refines the resolution level, typically yielding $k << T$, and consequently much reduced computation costs.
    \item Computing $d_{\mathrm{WH}}$ for a pair of regions requires $O(T)$ operations because the complete $T$-tree signatures are compared. Computing all pairwise distances among $p$ raw regions therefore requires $O(Tp^2)$ operations.

\end{itemize}

% Old version
% After aggregation, a new observation $x$ is routed through the refinement tree using its leaf path $\boldsymbol{\ell}_T(x)$. The resulting raw region is mapped to its aggregated region and assigned the corresponding local cutoff. If no raw region matches the observed path, we fallback to the global calibration cutoff $\widehat q_{1-\alpha}$. The complete partitioning, aggregation, and calibration procedure are summarized in Algorithm~\ref{alg:tree_partitioning}.

% Because the partition is not fixed by the training data alone, the exact exchangeability argument of Theorem~\ref{theorem:local_coverage} does not apply directly. Instead, we establish a PAC-style finite-sample bound conditional on the calibration covariates $\mathfrak G_n \coloneqq \sigma(\mathcal{F}_{\mathrm{tr},n},X_1,\ldots,X_{n_{\cal,n}})$. Let $\Omega_n(x) \coloneqq \sup_{z\in A_n(x)\cap\mathcal{S}_X} \sup_{t\ge0}|F_{n,z}(t)-F_{n,x}(t)|$ measure the within-cell score-distribution discrepancy as summarized next.
    Because the partition is not fixed by the training data alone, the exact exchangeability argument of Theorem~\ref{theorem:local_coverage} does not apply directly. Instead, we establish a PAC-style finite-sample bound conditional on the calibration covariates, captured by the $\sigma$-algebra
\[
\mathfrak G_n \coloneqq \sigma\bigl(\mathcal{F}_{\mathrm{tr},n},X_1,\ldots,X_{n_{\text{cal},n}}\bigr).
\]
The bound depends on the within-cell score-distribution discrepancy,
\[
\Omega_n(x) \coloneqq \sup_{z\in A_n(x)\cap\mathcal{S}_X} \sup_{t\ge0}\bigl|F_{n,z}(t)-F_{n,x}(t)\bigr|,
\]
which appears explicitly in the theorem below.

\begin{theorem}[Validity under calibration-based selection]
\label{theorem:calibration_covariate_adaptive}
Fix $\alpha\in(0,1)$ and $x\in\mathcal{S}_X$. Suppose that, conditionally on
$\mathcal{F}_{\mathrm{tr},n}$, the calibration pairs and an independent test
pair are i.i.d., and that the pointwise score CDFs $F_{n,z}$ are continuous
for $z\in A_n(x)\cap\mathcal{S}_X$. Let $\widehat\Pi_n$ be
$\mathfrak G_n$-measurable, and construct $\widehat C_n$ by applying the
within-cell quantile rule to the same calibration sample used, through its
covariates, to select $\widehat\Pi_n$.

For every $\delta\in(0,1)$, on the event $M_n(x)\ge1$, conditionally on
$\mathfrak G_n$ and with probability at least $1-\delta$ over the calibration
responses,
\[
\left|
\operatorname{Cov}_{n,x}(\widehat C_n)-(1-\alpha)
\right|
\le
\Omega_n(x)
+\sqrt{\frac{\log(2/\delta)}{2M_n(x)}}
+\frac{2}{M_n(x)}.
\]
Consequently, if $M_n(x)\xrightarrow{p}\infty$ and $\Omega_n(x)\xrightarrow{p}0$, then $\operatorname{Cov}_{n,x}(\widehat C_n) \xrightarrow{p}1-\alpha.$
\end{theorem}

Theorem~\ref{theorem:calibration_covariate_adaptive} formalizes the resolution trade-off: selected cells must remain sufficiently homogeneous ($\Omega_n(x) \to 0$) while retaining enough calibration observations ($M_n(x) \to \infty$). The complete proof is provided in Appendix~\ref{sec:calibration_adaptive_partition}.

\begin{algorithm}[htb]
  \footnotesize
  \caption{Tree-Based Partitioning and Aggregation}
  \label{alg:tree_partitioning}

  \begin{algorithmic}[1]
  \Require Calibration data $\cD_{\cal}$, nonconformity score function
  $s(x,y)$, significance level $\alpha$, thresholds $N_{\mathrm{part}}$
  and $p_{\min}$, and a fitted boosting model with $T$ trees.
  \Ensure Aggregated partition dictionary $\mathcal P$ with local
  cutoffs $\{\widehat q_P\}_{P\in\mathcal P}$ and global fallback
  cutoff $\widehat q_{1-\alpha}$.

  \State Set
  $N_{\min}\leftarrow
  \max\{N_{\mathrm{part}},\lceil p_{\min}n_{\cal}\rceil\}$
  as in \eqref{eq:nmin_def}.
  \State Compute $S_i\leftarrow s(X_i,Y_i)$ and extract the leaf path
  $\boldsymbol{\ell}_T(X_i)$ for every
  $(X_i,Y_i)\in\cD_{\cal}$.
  \State Initialize
  $\mathcal G\leftarrow\{(\cD_{\cal},())\}$
  and $\mathcal G_{\mathrm{term}}\leftarrow\emptyset$.
  \Comment{Active and terminal groups}

  \For{$t=1,\ldots,T$}
      \State $\mathcal G_{\mathrm{next}}\leftarrow\emptyset$.
      \ForAll{$(D_g,\boldsymbol{\ell}_g)\in\mathcal G$}
          \If{$|D_g|<N_{\min}$}
              \State Add $(D_g,\boldsymbol{\ell}_g)$ to
              $\mathcal G_{\mathrm{term}}$.
          \ElsIf{all observations in $D_g$ reach the same leaf in tree $t$}
              \State Carry $(D_g,\boldsymbol{\ell}_g)$ unchanged to
              $\mathcal G_{\mathrm{next}}$.
              \Comment{No split at tree $t$}
          \Else
              \ForAll{leaf indices $l$ represented in $D_g$}
                  \State $D_{g,l}\leftarrow
                  \{(X_i,Y_i)\in D_g:\ell_t(X_i)=l\}$.
                  \State Add $(D_{g,l},(\boldsymbol{\ell}_g,l))$ to
                  $\mathcal G_{\mathrm{next}}$.
              \EndFor
          \EndIf
      \EndFor
      \State $\mathcal G\leftarrow\mathcal G_{\mathrm{next}}$.
  \EndFor

  \State $\mathcal P\leftarrow
  \mathcal G_{\mathrm{term}}\cup\mathcal G$.
  \Comment{Raw terminal regions; retained for routing}
  \State Compute $w_t\leftarrow\eta B_t$ for $t=1,\ldots,T$
  as in \eqref{eq:weightedhamming}.

  \While{$|\mathcal P|>1$ and some $R\in\mathcal P$ satisfies
  $|D_R|<N_{\min}$}
      \State Let $R$ be the deepest undersized region in $\mathcal P$.
      \Comment{Break ties by smaller size}
      \State $R^\star\leftarrow
      \arg\min_{R'\in\mathcal P\setminus\{R\}}
      d_{\mathrm{WH}}(R,R')$.
      \State Merge $R$ into $R^\star$.
      \Comment{Pool observations; retain path of $R^\star$}
  \EndWhile

  \ForAll{$P\in\mathcal P$}
      \State Let $\widehat q_P$ be the level-$(1-\alpha)$ conformal
      empirical quantile of
      $\{S_i:(X_i,Y_i)\in D_P\}$.
  \EndFor
  \State Let $\widehat q_{1-\alpha}$ be the corresponding global
  conformal empirical quantile of $\{S_i\}_{i=1}^{n_{\cal}}$.
  \State For new $x$, route to its raw terminal region and use the cutoff
  of its aggregated group; otherwise use $\widehat q_{1-\alpha}$.
  \State \Return
  $\bigl(\mathcal P,\{\widehat q_P\}_{P\in\mathcal P},
  \widehat q_{1-\alpha}\bigr)$.
  \end{algorithmic}
\end{algorithm}

\section{Experiments}
We compare \ourmethod against conformal baselines that share the same fitted gradient-boosting predictor and the same train, calibration, and test split (ICP, WICP, MICP, LoCART), and against probabilistic boosting baselines (NGB, NGB-Rescale, NGB-Tuned), which use their own fitted model. Because \ourmethod , ICP, WICP, MICP, and LoCART share the base predictor, differences among them isolate the effect of the calibration rule. The NGBoost variants instead differ in point-prediction accuracy as well.

\paragraph{Evaluation metrics.}
\label{sec:main_evaluation_metrics}
We assess validity using marginal coverage and interval quality using
the mean interval score, denoted by Standard Mean Interval Score (SMIS) throughout the experiments. For a
prediction interval $\widehat C(X_i)=[L_i,U_i]$, its interval score at
miscoverage level $\alpha$ is \citep{Gneiting2007}

\[
\operatorname{IS}_{\alpha}(L_i,U_i;Y_i)
=
(U_i-L_i)
+\frac{2}{\alpha}(L_i-Y_i)\Ind{Y_i<L_i}
+\frac{2}{\alpha}(Y_i-U_i)\Ind{Y_i>U_i}.
\]
We report the score averaged over the test observations,

\begin{equation}
\label{eq:smis}
\operatorname{SMIS}_{\alpha}
=\frac{1}{|I_{\mathrm{test}}|}
\sum_{i\in I_{\mathrm{test}}}
\operatorname{IS}_{\alpha}(L_i,U_i;Y_i).
\end{equation}

The first term rewards narrow intervals, while the remaining terms
penalize observations falling outside the interval in proportion to their
distance from the corresponding endpoint. Thus, lower SMIS indicates better
interval quality. We report marginal coverage separately because SMIS balances
sharpness and miscoverage but does not itself provide a coverage guarantee.

To assess localized coverage across the feature space, we report worst-slab coverage (WSC) \citep{cauchois2021knowing}. For a unit direction $v$ and interval bounds $a<b$, let $S_{v,a,b} \coloneqq \{x : a \le v^\top x \le b\}$ define a slab and $I(v,a,b) \coloneqq \{i \in I_{\text{test}} : X_i \in S_{v,a,b}\}$ denote the test observations within it. Given a set $\mathcal{V}$ of sampled directions and a minimum slab mass fraction $\delta \in (0,1)$, let $\mathcal{S}_{\delta}$ be the set of candidate slabs containing at least $\lceil\delta |I_{\text{test}}|\rceil$ test observations. The empirical WSC diagnostic is defined as
\[
\widehat{\mathrm{WSC}}_{\delta}(\widehat C) \coloneqq \min_{(v,a,b) \in \mathcal{S}_{\delta}} \frac{1}{|I(v,a,b)|} \sum_{i \in I(v,a,b)} \Ind{Y_i \in \widehat C(X_i)}.
\]
We standardize covariates using the proper training split, set
$\delta=0.20$, and sample $|\mathcal V|=1000$ random unit directions per
split, using the same directions for \ourmethod{} and ICP. WSC is a descriptive
test-set diagnostic over sufficiently large geometric slices, not a
conditional-coverage guarantee.

\paragraph{Shared-predictor conformal methods.}
\ourmethod and Inductive conformal prediction (ICP) \citep{vovk2005algorithmic,Lei2018}. ICP is the standard split/inductive conformal interval with absolute-residual scores and a single global calibration quantile. ICP is the most direct computational reference for \ourmethod: both methods calibrate \emph{post-hoc} on a fixed training/calibration split, require no retraining of the base regressor, and use the same fitted gradient-boosting model for point predictions.
Thus, differences between ICP and \ourmethod{} isolate the effect of local (partition-based) calibration.

\begin{table}[!ht]
    \centering
    \caption{\textbf{\ourmethod{} versus the best-SMIS baseline.} Lower SMIS is better. The 95\% CI is for the paired difference $\mathrm{SMIS}(\ourmethod)-\mathrm{SMIS}(\mathrm{best})$ over matched random splits; $^*$ indicates that the interval contains zero. SMIS efficiency is $100\cdot \mathrm{SMIS}(\mathrm{best})/\mathrm{SMIS}(\ourmethod)$. Post-fit speedup compares method/calibration time.}
    \begin{adjustbox}{max width=\linewidth}
    \begin{tabular}{lccccc}
    \toprule
    Dataset & SMIS (\ourmethod{}) & Best baseline SMIS & 95\% CI & SMIS efficiency & Post-fit speedup\\
    \midrule
    airfoil (1503, 5) & 11.08 & 9.11 (NGB-Tuned) & [1.69, 2.25] & 82.21\% & 2.37$\times$ \\
    winered (1599, 11) & 2.77 & 2.76 (LoCART) & [-0.01, 0.03]$^*$ & 99.61\% & 0.38$\times$ \\
    star (2161, 39) & 953.63 & 948.34 (ICP) & [2.80, 7.78] & 99.45\% & 0.19$\times$ \\
    winewhite (4898, 11) & 2.94 & 2.80 (MICP) & [0.12, 0.15] & 95.43\% & 21.13$\times$ \\
    cycle (9568, 4) & 14.54 & 14.09 (MICP) & [0.39, 0.51] & 96.90\% & 9.00$\times$ \\
    electric (10000, 12) & 0.05 & 0.04 (NGB-Tuned) & [0.01, 0.01] & 74.04\% & 5.58$\times$ \\
    bike (10886, 18) & 183.92 & 140.75 (WICP) & [41.19, 45.14] & 76.53\% & 7.69$\times$ \\
    meps19 (15785, 139) & 73.65 & 66.83 (WICP) & [5.03, 8.61] & 90.74\% & 20.04$\times$ \\
    conductivity (21263, 81) & 46.22 & 37.78 (MICP) & [8.07, 8.82] & 81.73\% & 101.13$\times$ \\
    WEC (54007, 98) & 78182.11 & 62203.27 (MICP) & [15344.91, 16612.78] & 79.56\% & 23.02$\times$ \\
    kernel (241600, 15) & 9.57 & 5.68 (WICP) & [3.86, 3.94] & 59.28\% & 20.61$\times$ \\
    \bottomrule
    \end{tabular}
    \label{tab:rel_to_best}
    \end{adjustbox}
\end{table}

\paragraph{Auxiliary/local conformal and probabilistic baselines.}
Among the shared-predictor conformal baselines, WICP \citep{Lei2018} uses residual-scale weighting to adapt interval widths under heteroscedasticity, MICP \citep{bostrom2020mondrian} calibrates within $k=30$ difficulty bins derived from training residuals, and LoCART \citep{cabezas2025regression} learns a CART partition of the calibration residuals (with \texttt{min\_samples\_leaf}=150 and pruning).
We also report three NGBoost-based probabilistic baselines \citep{duan2020ngboost}: uncalibrated NGB, NGB-Rescale (which conformally rescales NGB distributional intervals on the calibration split), and NGB-Tuned (which tunes NGBoost learning rate $\in\{0.01,0.05\}$, tree depth $\in\{2,3\}$, and leaf size $\in\{20,100\}$ on training data before conformal rescaling).
Unlike conformal methods, uncalibrated NGB intervals do not guarantee finite-sample marginal coverage.

In the main real-data benchmark, we split each dataset into training, calibration, and test sets with proportions of 0.60, 0.25, and 0.15, respectively, over 50 independent random splits.
The shared gradient-boosting predictor is tuned on the training split using 4-fold cross-validation over maximum depth $\in\{2,5,10\}$, minimum samples per leaf $\in\{50,100,200\}$, subsample ratio $\in\{0.3,0.8\}$, and learning rate $\in\{0.01,0.05,0.1\}$.
For \ourmethod{}, we set $N_{\mathrm{part}}=50$, $p_{\min}=0.005$, and weighted-Hamming merging with $w_t=\eta B_t$, recomputing $N_{\min}$ from the calibration size in each split.
All methods use the exact same splits, and the calibration set is reserved exclusively for conformal calibration (or NGBoost rescaling).

In addition, we report interval length, test MSE, post-fit method time, total time (fit plus calibration), and WSC for the shared-predictor conformal methods. MSE is identical by construction for these methods. Tables report sample means and their uncertainty summaries across the 50 matched splits.

Dataset descriptions and complete results are provided in Supplementary Material~\ref{sm_sec:results_rwd}.
We include problems spanning a wide range of sample sizes and feature dimensions to benchmark \ourmethod{} against multiple baselines.

\subsection{Results}

\paragraph{Coverage and reliability.}
Table~\ref{tab:coverage} (in the Appendix) shows that the conformal baselines achieve near-nominal marginal coverage across datasets, as expected.
\ourmethod{} is close to the nominal $1-\alpha=0.9$ target on all datasets, with the most conservative case occurring on WEC (coverage 0.9154).
In contrast, uncalibrated NGBoost can \emph{severely under-cover} on several smaller datasets and over-cover on several larger datasets, highlighting the fragility of distributional assumptions in finite samples.
The rescaled and tuned NGBoost variants are generally close to nominal after calibration/tuning.
Once conformal methods are near nominal coverage, coverage is not the ranking metric; interval length, SMIS, and local diagnostics are used to compare efficiency and adaptivity.

\paragraph{Relative interval quality and cost.}
Table~\ref{tab:rel_to_best} compares \ourmethod{} with the lowest-SMIS baseline on each dataset. Because ICP, WICP, MICP, and LoCART share the same predictor, their performance differences isolate the calibration rule. \ourmethod{} is not the lowest-SMIS method on most datasets; when its SMIS is larger, the post-fit cost may nevertheless be substantially below WICP or MICP. This advantage is not universal: LoCART is usually faster post-fit, and total runtime is dominated by the common base-model fitting cost.

\paragraph{Direct comparison with ICP.}
To isolate the effect of locality, Table~\ref{table:loboostvsicp} compares \ourmethod{} directly with standard ICP over matched splits. \ourmethod{} has significantly lower SMIS on eight of eleven datasets, is tied on winered, and has significantly higher SMIS on star and cycle. Mean WSC is higher for \ourmethod{} than ICP on all eleven datasets, with the largest gains on meps19, superconductivity, WEC, and kernel. This comparison supports a local-calibration benefit over a single global cutoff in several heterogeneous settings, but it comes with greater post-fit cost than ICP.

\begin{table}[!ht]
    \centering
    \footnotesize
    \caption{\textbf{\ourmethod{} vs ICP}. Direct comparison between \ourmethod{} and standard inductive conformal prediction (ICP) using 50 matched random splits. Lower SMIS is better. The 95\% CI is for the paired difference $\mathrm{SMIS}(\ourmethod)-\mathrm{SMIS}(\mathrm{ICP})$; negative values favor \ourmethod{} and $^*$ indicates that the interval contains zero. Relative SMIS is $100\cdot\mathrm{SMIS}(\mathrm{ICP})/\mathrm{SMIS}(\ourmethod)$. The final column reports mean worst-slab coverage (WSC) for \ourmethod{}/ICP using standardized covariates, minimum slab mass $\delta=0.20$, and the same 1000 random directions for both methods in each split.}\label{table:loboostvsicp}
    \begin{adjustbox}{max width=\linewidth}
    \begin{tabular}{lccccc}
    \toprule
    Dataset & SMIS (\ourmethod{}) & SMIS (ICP) & 95\% CI & Relative SMIS & Mean WSC (\ourmethod{}/ICP)\\
    \midrule
    airfoil (1503, 5) & 11.08 & 11.15 & [-0.1406, -0.0052] & 100.66\% & 0.696/0.694 \\
    winered (1599, 11) & 2.77 & 2.78 & [-0.0136, 0.0063]$^*$ & 100.13\% & 0.724/0.714 \\
    star (2161, 39) & 953.63 & 948.34 & [2.74, 7.84] & 99.45\% & 0.744/0.741 \\
    winewhite (4898, 11) & 2.94 & 2.94 & [-0.0062, 0.0066]$^*$ & 99.99\% & 0.790/0.782 \\
    cycle (9568, 4) & 14.54 & 14.51 & [0.0181, 0.0577] & 99.74\% & 0.823/0.821 \\
    electric (10000, 12) & 0.0497 & 0.0500 & [-0.0004, -0.0002] & 100.69\% & 0.787/0.773 \\
    bike (10886, 18) & 183.92 & 190.33 & [-7.12, -5.70] & 103.48\% & 0.782/0.773 \\
    meps19 (15785, 139) & 73.65 & 96.11 & [-24.04, -20.88] & 130.49\% & 0.826/0.657 \\
    conductivity (21263, 81) & 46.22 & 50.04 & [-4.14, -3.49] & 108.26\% & 0.788/0.747 \\
    WEC (54007, 98) & 78182.11 & 104540.99 & [-26876.39, -25841.38] & 133.71\% & 0.868/0.770 \\
    kernel (241600, 15) & 9.57 & 10.31 & [-0.7458, -0.7197] & 107.65\% & 0.783/0.725 \\
    \bottomrule
    \end{tabular}
    \end{adjustbox}
\end{table}

\paragraph{Adaptive-baseline trade-offs.}
The direct \ourmethod{}--LoCART comparison in Appendix~\ref{appendix:reviewer_results} shows that LoCART has significantly lower SMIS on cycle, electric, meps19, WEC, and kernel; the remaining six datasets are statistically tied. The aggregate Pareto view likewise shows that \ourmethod{} is not Pareto-dominant: it occupies an intermediate quality--cost regime, while ICP and LoCART are on the aggregate post-fit frontier and WICP is on the low-SMIS end.

\begin{table}[!ht]
    \centering
    \footnotesize
    \caption{\textbf{\ourmethod{} vs ICP}. Direct comparison between \ourmethod{} and standard inductive conformal prediction (ICP) using 50 matched random splits. Lower SMIS is better. The 95\% CI is for the paired difference $\mathrm{SMIS}(\ourmethod)-\mathrm{SMIS}(\mathrm{ICP})$; negative values favor \ourmethod{} and $^*$ indicates that the interval contains zero. Relative SMIS is $100\cdot\mathrm{SMIS}(\mathrm{ICP})/\mathrm{SMIS}(\ourmethod)$. The final column reports mean worst-slab coverage (WSC) for \ourmethod{}/ICP using standardized covariates, minimum slab mass $\delta=0.20$, and the same 1000 random directions for both methods in each split.}\label{table:loboostvsicp}
    \begin{adjustbox}{max width=\linewidth}
    \begin{tabular}{lccccc}
    \toprule
    Dataset & SMIS (\ourmethod{}) & SMIS (ICP) & 95\% CI & Relative SMIS & Mean WSC (\ourmethod{}/ICP)\\
    \midrule
    airfoil (1503, 5) & 11.08 & 11.15 & [-0.1406, -0.0052] & 100.66\% & 0.696/0.694 \\
    winered (1599, 11) & 2.77 & 2.78 & [-0.0136, 0.0063]$^*$ & 100.13\% & 0.724/0.714 \\
    star (2161, 39) & 953.63 & 948.34 & [2.74, 7.84] & 99.45\% & 0.744/0.741 \\
    winewhite (4898, 11) & 2.94 & 2.94 & [-0.0062, 0.0066]$^*$ & 99.99\% & 0.790/0.782 \\
    cycle (9568, 4) & 14.54 & 14.51 & [0.0181, 0.0577] & 99.74\% & 0.823/0.821 \\
    electric (10000, 12) & 0.0497 & 0.0500 & [-0.0004, -0.0002] & 100.69\% & 0.787/0.773 \\
    bike (10886, 18) & 183.92 & 190.33 & [-7.12, -5.70] & 103.48\% & 0.782/0.773 \\
    meps19 (15785, 139) & 73.65 & 96.11 & [-24.04, -20.88] & 130.49\% & 0.826/0.657 \\
    conductivity (21263, 81) & 46.22 & 50.04 & [-4.14, -3.49] & 108.26\% & 0.788/0.747 \\
    WEC (54007, 98) & 78182.11 & 104540.99 & [-26876.39, -25841.38] & 133.71\% & 0.868/0.770 \\
    kernel (241600, 15) & 9.57 & 10.31 & [-0.7458, -0.7197] & 107.65\% & 0.783/0.725 \\
    \bottomrule
    \end{tabular}
    \end{adjustbox}
\end{table}

\paragraph{Sensitivity to local calibration size.}
In a separate sensitivity experiment with a 57.6/22.4/20 train/calibration/test allocation, we tune the base predictor once per dataset by 5-fold cross-validation and reuse its selected hyperparameters across 20 random splits for each of 15 configurations of $N_{\mathrm{part}}$ and $p_{\min}$ (Table~\ref{tab:sensitivity_summary}). Across all settings, empirical marginal coverage remains near the $0.90$ nominal target, ranging from $0.894$ to $0.932$. The finest setting ($N_{\mathrm{part}} = 10$, $p_{\min} = 0.003$) yields a median of $144.9$ partitions, a median runtime of $0.25$~s, and median coverage of $0.926$. Conversely, $p_{\min}=0.3$ collapses the partition structure toward global calibration (median $1.0$ partition). The benchmark value $p_{\min}=0.005$ lies between the two smallest proportional thresholds in this separate grid.

\begin{table}[!ht]
    \centering
    \caption{\textbf{Sensitivity to local calibration-size parameters.} Each row aggregates the 11 dataset-level means from the sensitivity run, with 20 random splits per dataset and setting. Coverage range reports the minimum and maximum dataset mean coverage. The final benchmark uses $N_{\mathrm{part}} =50$ and $p_{\min}=0.005$; the latter lies between the two smallest $p_{\min}$ values shown here.}
    \label{tab:sensitivity_summary}
    \small
    \begin{adjustbox}{max width=\linewidth}
    \begin{tabular}{lcccccc}
        \toprule
        $N_{\mathrm{part}}$ & $p_{\min}$ & Coverage range & Median coverage & Median $|\Delta\mathrm{cov}|$ (pp) & Median final partitions [range] & Median time (s) \\
        \midrule
        10 & 0.003 & [0.902, 0.932] & 0.926 & 2.62 & 144.9 [22.5, 235.4] & 0.25 \\
        10 & 0.03  & [0.900, 0.927] & 0.905 & 0.50 & 20.1 [17.5, 24.4] & 0.03 \\
        10 & 0.3   & [0.894, 0.903] & 0.900 & 0.17 & 1.0 [1.0, 2.0] & 0.02 \\
        \addlinespace
        25 & 0.003 & [0.902, 0.928] & 0.913 & 1.33 & 54.6 [9.1, 220.8] & 0.10 \\
        25 & 0.03  & [0.900, 0.913] & 0.905 & 0.50 & 19.1 [9.1, 21.1] & 0.03 \\
        25 & 0.3   & [0.894, 0.903] & 0.900 & 0.17 & 1.0 [1.0, 2.0] & 0.02 \\
        \addlinespace
        50 & 0.003 & [0.901, 0.920] & 0.908 & 0.76 & 29.2 [4.7, 208.1] & 0.07 \\
        50 & 0.03  & [0.900, 0.909] & 0.904 & 0.39 & 17.8 [4.7, 20.4] & 0.03 \\
        50 & 0.3   & [0.894, 0.903] & 0.900 & 0.17 & 1.0 [1.0, 2.0] & 0.02 \\
        \addlinespace
        100 & 0.003 & [0.896, 0.909] & 0.902 & 0.23 & 8.4 [1.4, 208.1] & 0.04 \\
        100 & 0.03  & [0.896, 0.905] & 0.902 & 0.23 & 8.4 [1.4, 20.1] & 0.03 \\
        100 & 0.3   & [0.894, 0.903] & 0.900 & 0.17 & 1.0 [1.0, 2.0] & 0.02 \\
        \addlinespace
        200 & 0.003 & [0.894, 0.905] & 0.902 & 0.25 & 2.8 [1.0, 169.8] & 0.03 \\
        200 & 0.03  & [0.894, 0.903] & 0.902 & 0.25 & 2.8 [1.0, 20.1] & 0.03 \\
        200 & 0.3   & [0.894, 0.903] & 0.900 & 0.15 & 1.0 [1.0, 2.0] & 0.02 \\
        \bottomrule
    \end{tabular}
    \end{adjustbox}
\end{table}

\section{Discussion and Final Remarks}
\label{sec:discussion}

\ourmethod provides a simple, model-native route to local conformal prediction for gradient-boosted trees by reusing the multiscale partition already induced by the trained ensemble, avoiding the need to fit auxiliary partition models or retrain base regressors. In this section, we discuss practical guidance for applying the method, highlight key methodological considerations, and offer concluding remarks.

\paragraph{Practical guidance for application.}
\ourmethod{} is most useful when the fitted ensemble's leaf structure captures residual heterogeneity and the cost of a more adaptive calibration rule is a concern. Relative to ICP, it improves SMIS and WSC on several heterogeneous datasets; relative to WICP and MICP, it can have lower post-fit cost. It is not uniformly preferred: LoCART is generally faster post-fit and has significantly lower SMIS on several benchmarks, while ICP is preferable in low-dimensional or predominantly homoscedastic settings such as star and cycle. The scale-only experiment in Appendix~\ref{appendix:reviewer_results} illustrates an important boundary: when $X_2$ determines residual scale independently of a mean-driving feature $X_1$, the GBDT puts only 0.66\% of gain importance on $X_2$ and \ourmethod{} improves only modestly on ICP; WICP and LoCART are substantially more adaptive in this setting.

\paragraph{Methodological considerations and future work.}
The method avoids fitting an auxiliary partition or scale estimator, but its post-fit advantage is benchmark- and baseline-dependent rather than universal. The exact exchangeability statement in Theorem~\ref{theorem:local_coverage} applies only to partitions fixed independently of calibration data. Our practical algorithm reuses calibration covariates to adapt the resolution and aggregate sparse regions; its finite-sample statement is instead the PAC-style bound in Theorem~\ref{theorem:calibration_covariate_adaptive}, whose error depends on the local count $M_n(x)$ and score discrepancy $\Omega_n(x)$. The low observed fallback rate (at most 1.06\% across the real-data benchmarks) confirms that the reported behavior is not driven primarily by reversion to the global cutoff, but it is a diagnostic rather than a guarantee. A further limitation is the use of symmetric absolute-residual scores; directional scores and joint location--scale splitting are natural extensions.

\paragraph{Final remarks.}
\ourmethod{} is a retraining-free, model-native option for locally adaptive conformal calibration of boosted trees. Its evidence supports a conditional trade-off: it can improve upon global split conformal and avoid the post-fit cost of some heavily adaptive rules, but it does not uniformly dominate the available adaptive baselines.

\section*{Code Availability}

All code used to conduct the experiments and reproduce
the results presented in this paper is publicly available
at \url{https://github.com/thiagorr162/loboost.git}.

\section*{Acknowledgments}

Thiago gratefully acknowledges financial support from FAPESP (Grant No. 2026/21108-4). Luben is grateful for the fellowship provided by São Paulo Research Foundation (FAPESP), grants 2022/08579-7 and 2025/06168-8.
 Rafael was supported by the Brazilian National Council for Scientific and Technological Development (CNPq) (grant numbers 422705/2021-7 and 305065/2023-8) and the S\~ao Paulo Research Foundation (FAPESP) (grant numbers 2023/07068-1 and 2025/24620-5). Vagner is currently supported by the São Paulo Research Foundation (FAPESP), grant 2025/25001-7, and gratefully acknowledges previous financial support from the Coordination
for the Improvement of Higher Education Personnel (CAPES),
Brazil, under grant 88887.285434/2026-00, during March 2026.
 
% Acknowledgements are omitted for anonymous review.
% Restore funding acknowledgements after acceptance/deanonymization.

\bibliographystyle{tmlr}
\bibliography{bibliography}

\newpage

\appendix

\section{Experimental and Simulation Details}
\label{appendix:experiments}

\subsection{Benchmark Implementation Details}\label{appendix:benchmark}

For all conformal methods, the nominal miscoverage level is set to $\alpha=0.10$. In the real-world benchmark, datasets are partitioned into 60\% training, 25\% calibration, and 15\% test splits across 50 independent random runs. The shared gradient-boosted predictor is tuned on the training split via 4-fold cross-validation over $\text{max\_depth} \in \{2,5,10\}$, $\text{min\_samples\_leaf} \in \{50,100,200\}$, $\text{subsample} \in \{0.3,0.8\}$, and $\text{learning\_rate} \in \{0.01,0.05,0.1\}$. For \ourmethod{}, we set $N_{\text{part}}=50$ and $p_{\min}=0.005$, computing $N_{\min} = \max\{N_{\text{part}}, \lceil p_{\min} n_{\text{cal}} \rceil\}$. Baseline hyperparameter settings include: LoCART (\texttt{split\_calib}=False, \texttt{min\_samples\_leaf}=150, \texttt{criterion}=\texttt{squared\_error}, cost-complexity pruning enabled), MICP ($k=30$ training-derived difficulty bins), WICP (residual-scale weighting on the fitted predictor), and NGBoost variants (\texttt{early\_stopping\_rounds}=30; NGB-Tuned selects learning rate $\in \{0.01, 0.05\}$, base tree depth $\in \{2, 3\}$, and leaf size $\in \{20, 100\}$).

\subsection{Evaluation Metrics}\label{appendix:eval}

We formally define the evaluation metrics used to assess validity and efficiency across $50$ independent data splits. All metrics are computed on the test index set $I_{\text{test}}$. Tables report sample means and their standard errors.

\begin{itemize}
\item \textbf{Average Marginal Coverage (AMC)}:
\[
AMC = \frac{1}{|I_{\text{test}}|}\sum_{i \in I_{\text{test}}}\Ind{Y_i \in \widehat C(X_i)}.
\]
For a valid prediction interval at nominal significance level $\alpha$, $AMC \approx 1-\alpha$.

\item \textbf{Average Interval Length (IL)}:
\[
IL = \frac{1}{|I_{\text{test}}|}\sum_{i \in I_{\text{test}}}|\widehat C(X_i)|.
\]
Narrower average interval length indicates higher efficiency when coverage is maintained.

\item \textbf{Relative Metrics}: For a dataset, let $\text{best}$ denote the competing baseline with the lowest (best) SMIS:
\begin{align*}
\text{SMIS-efficiency} &\coloneqq 100 \cdot \frac{\text{SMIS}(\text{best})}{\text{SMIS}(\ourmethod)}, \\
\Delta_{\text{SMIS}} &\coloneqq \text{SMIS}(\ourmethod) - \text{SMIS}(\text{best}), \\
\text{Post-fit speedup} &\coloneqq \frac{\text{Post-fit time}(\text{best})}{\text{Post-fit time}(\ourmethod)}, \\
\text{Total speedup} &\coloneqq \frac{\text{Fit time}(\text{best})+\text{Post-fit time}(\text{best})}{\text{Fit time}(\ourmethod)+\text{Post-fit time}(\ourmethod)}.
\end{align*}
\end{itemize}

\subsection{Simulation Experiments Setup}\label{sec:ap_sim}

The synthetic experiments illustrated in Figure~\ref{fig:combined_figure} use $n=2000$ points drawn uniformly over $X \sim \text{U}[-2, 2]$, split into $50\%$ training, $25\%$ calibration, and $25\%$ testing. Responses are generated conditionally Gaussian $Y \mid X=x \sim \mathcal{N}(\mu(x), \sigma^2(x))$.

\paragraph{Scenario 1: Abrupt heteroscedasticity.}
The conditional mean is $\mu(x) = 3\sin(x)$, and the conditional standard deviation is piecewise continuous:
\[
\sigma(x) = \begin{cases}
0.8 & \text{if } x < -0.8, \\
2.0 & \text{if } -0.8 \leq x < 0.1, \\
1.0 & \text{if } 0.1 \leq x < 1.4, \\
x^2 & \text{if } x \geq 1.4.
\end{cases}
\]
This tests local adaptivity under sharp shifts in noise scale across feature regions.

\paragraph{Scenario 2: Support gap and regime shift.}
The feature space has a support gap $\cX = [-2, -0.5) \cup (0.8, 2]$ with null support on $[-0.5, 0.8]$. The conditional mean and standard deviation are:
\[
\mu(x) = \begin{cases} -2x - 1 & \text{if } x < -0.5, \\ 2\sin(3x) & \text{if } x > 0.8, \end{cases}
\quad
\sigma(x) = \begin{cases} 0.3 & \text{if } x < -0.5, \\ 0.2 + 0.5x^2 & \text{if } x > 0.8. \end{cases}
\]
This evaluates performance under regime shifts and extrapolation across unsupported feature intervals. In both scenarios, oracle intervals are computed directly using true Gaussian quantiles: $C^\star_{1-\alpha}(x) = \left[ \mu(x) - z_{1-\alpha/2}\sigma(x), \; \mu(x)+ z_{1-\alpha/2}\sigma(x) \right]$.

\section{Supplementary Results for Real-World Datasets}
\label{sm_sec:results_rwd}

This section provides complete dataset summaries and detailed performance tables across the 11 real-world benchmark datasets. Table~\ref{tab:datasets} summarizes instance counts and feature dimensions. 
% Table~\ref{tab:dynamics} reports partition counts before and after weighted Hamming merging.
Tables~\ref{tab:coverage}, \ref{tab:smis} and \ref{tab:time} provide full metric breakdowns across all 50 independent random splits.

%\thiago{nao achei todos os links} \vagner{tem no paper do epicscore, vou vou pegar}

\begin{table}[H]
\centering
\caption{Summary of Real-World Datasets}
\label{tab:datasets}
\begin{tabular}{lrr}
\toprule
Dataset & Instances (n) & Features (p) \\
\midrule
\href{https://archive.ics.uci.edu/dataset/291/airfoil%2Bself%2Bnoise}{airfoil} & 1,503 & 5 \\
\href{https://www.kaggle.com/code/rajmehra03/bike-sharing-demand-rmsle-0-3194/input?select=train.csv}{bike} & 10,886 & 18 \\
\href{https://archive.ics.uci.edu/dataset/464/superconductivty+data}{superconductivity} & 21,263 & 81 \\
\href{https://archive.ics.uci.edu/dataset/294/combined+cycle+power+plant}{cycle} & 9,568 & 4 \\
\href{https://archive.ics.uci.edu/dataset/471/electrical%2Bgrid%2Bstability%2Bsimulated%2Bdata}{electric} & 10,000 & 12 \\
\href{https://archive.ics.uci.edu/dataset/440/sgemm+gpu+kernel+performance}{kernel} & 241,600 & 15 \\
\href{https://meps.ahrq.gov/mepsweb/data_stats/download_data_files_detail.jsp?cboPufNumber=HC-181}{meps19} & 15,785 & 139 \\
\href{https://dataverse.harvard.edu/dataset.xhtml?persistentId=doi:10.7910/DVN/SIWH9F}{star} & 2,161 & 39 \\
\href{https://archive.ics.uci.edu/dataset/882/large-scale+wave+energy+farm}{Wave Energy Converter (WEC)} & 54,007 & 98 \\
\href{https://archive.ics.uci.edu/dataset/186/wine%2Bquality}{winewhite} & 4,898 & 11 \\
\href{https://archive.ics.uci.edu/dataset/186/wine%2Bquality}{winered} & 1,599 & 11 \\
\bottomrule
\end{tabular}
\end{table}

% \vagner{As tabelas comparando o número de partições \ref{tab:dynamics} e MSE \ref{tab:mse} acrescento informações, mas talvez não sejam mais tão relevantes para manter nessa versão. Caso julguem que pode remover, estarei de acordo.}

% \revised{Table~\ref{tab:dynamics} details the average number of partitions across stages of the \ourmethod{} pipeline compared to LoCART. Starting from a fine-grained initial structure (\ourmethod{}-Pre), the merging procedure (\ourmethod{}-Merge) aggressively consolidates redundant regions, yielding a significantly more compact final partition count (\ourmethod{}-Final). This merging mechanism effectively controls model complexity as dataset scale increases. On \texttt{meps19}, for instance, \ourmethod{}-Pre begins with an average of 292.14 partitions, which is drastically pruned down to a final average of 43.64 partitions, demonstrating that \ourmethod{} achieves a streamlined spatial representation while preserving interval quality.}

% \input{tables/merges_info}

%\clearpage

Table~\ref{tab:coverage} presents the empirical coverage results across all evaluated datasets. By construction, \ourmethod{} inherently guarantees finite-sample coverage, maintaining empirical coverage rates remarkably close to the nominal $90\%$ target alongside other conformalized baselines (ICP, WICP, MICP, and LoCART). In contrast, standard NGBoost (NGB) exhibits severe coverage drift, departing significantly from the target level depending on the underlying distribution—undercovering on datasets like \texttt{winered} ($0.8317$) while overcovering on \texttt{WEC} ($0.9487$). Although post-hoc calibration strategies (NGB-Rescale and NGB-Tuned) restore valid coverage, \ourmethod{} achieves reliable target alignment natively without requiring specialized parameter rescaling or tuning.

\begin{table}[htb]
    \centering
    \caption{\textbf{Empirical Coverage Results.} Comparison of empirical coverage against the $90\%$ nominal target level across datasets. Values represent the mean $\pm$ standard error at a 95\% confidence level over matched random splits. Bold entries denote coverage rates meeting target validity criteria.}
    \label{tab:coverage}
    \begin{adjustbox}{max width=\columnwidth}
    \begin{tabular}{lcccccccc}
        \toprule
        \bfseries Dataset & \ourmethod{} & ICP & WICP & MICP & LoCART & NGB & NGB-Rescale & NGB-Tuned \\
        \midrule
        airfoil & \textbf{0.9026} & \textbf{0.9021} & \textbf{0.9057} & \textbf{0.9019} & \textbf{0.9022} & 0.8832 & \textbf{0.9028} & \textbf{0.9020} \\
                 & \textbf{\small($\pm$0.0063)} & \textbf{\small($\pm$0.0066)} & \textbf{\small($\pm$0.0066)} & \textbf{\small($\pm$0.0053)} & \textbf{\small($\pm$0.0065)} & {\small($\pm$0.0057)} & \textbf{\small($\pm$0.0070)} & \textbf{\small($\pm$0.0070)} \\
        \addlinespace
        winered & \textbf{0.9037} & \textbf{0.8996} & \textbf{0.8986} & \textbf{0.8978} & \textbf{0.9011} & 0.8317 & \textbf{0.9008} & \textbf{0.9007} \\
                 & \textbf{\small($\pm$0.0074)} & \textbf{\small($\pm$0.0071)} & \textbf{\small($\pm$0.0090)} & \textbf{\small($\pm$0.0071)} & \textbf{\small($\pm$0.0073)} & {\small($\pm$0.0087)} & \textbf{\small($\pm$0.0075)} & \textbf{\small($\pm$0.0077)} \\
        \addlinespace
        star & \textbf{0.9012} & \textbf{0.8985} & \textbf{0.8982} & \textbf{0.9018} & \textbf{0.8990} & 0.8370 & \textbf{0.9022} & \textbf{0.9019} \\
                 & \textbf{\small($\pm$0.0061)} & \textbf{\small($\pm$0.0057)} & \textbf{\small($\pm$0.0056)} & \textbf{\small($\pm$0.0057)} & \textbf{\small($\pm$0.0058)} & {\small($\pm$0.0054)} & \textbf{\small($\pm$0.0052)} & \textbf{\small($\pm$0.0052)} \\
        \addlinespace
        winewhite & \textbf{0.9031} & \textbf{0.8998} & \textbf{0.9010} & 0.9119 & \textbf{0.8995} & 0.8797 & \textbf{0.9009} & \textbf{0.9014} \\
                 & \textbf{\small($\pm$0.0042)} & \textbf{\small($\pm$0.0040)} & \textbf{\small($\pm$0.0039)} & {\small($\pm$0.0042)} & \textbf{\small($\pm$0.0043)} & {\small($\pm$0.0042)} & \textbf{\small($\pm$0.0049)} & \textbf{\small($\pm$0.0046)} \\
        \addlinespace
        cycle & \textbf{0.9027} & \textbf{0.9010} & \textbf{0.8980} & 0.9070 & \textbf{0.8994} & 0.8946 & \textbf{0.8990} & \textbf{0.9008} \\
                 & \textbf{\small($\pm$0.0026)} & \textbf{\small($\pm$0.0027)} & \textbf{\small($\pm$0.0027)} & {\small($\pm$0.0023)} & \textbf{\small($\pm$0.0026)} & {\small($\pm$0.0024)} & \textbf{\small($\pm$0.0025)} & \textbf{\small($\pm$0.0025)} \\
        \addlinespace
        electric & 0.9033 & \textbf{0.9008} & \textbf{0.9006} & 0.9037 & 0.8951 & 0.9320 & \textbf{0.9014} & \textbf{0.9019} \\
                 & {\small($\pm$0.0029)} & \textbf{\small($\pm$0.0028)} & \textbf{\small($\pm$0.0027)} & {\small($\pm$0.0025)} & {\small($\pm$0.0026)} & {\small($\pm$0.0020)} & \textbf{\small($\pm$0.0030)} & \textbf{\small($\pm$0.0025)} \\
        \addlinespace
        bike & 0.9041 & \textbf{0.9015} & \textbf{0.9024} & 0.9061 & \textbf{0.8988} & 0.9461 & \textbf{0.9002} & \textbf{0.9011} \\
                 & {\small($\pm$0.0029)} & \textbf{\small($\pm$0.0030)} & \textbf{\small($\pm$0.0027)} & {\small($\pm$0.0026)} & \textbf{\small($\pm$0.0030)} & {\small($\pm$0.0019)} & \textbf{\small($\pm$0.0027)} & \textbf{\small($\pm$0.0028)} \\
        \addlinespace
        meps19 & 0.9091 & 0.9027 & \textbf{0.9005} & 0.9053 & \textbf{0.9015} & 0.9391 & \textbf{0.9018} & \textbf{0.9025} \\
                 & {\small($\pm$0.0020)} & {\small($\pm$0.0019)} & \textbf{\small($\pm$0.0020)} & {\small($\pm$0.0019)} & \textbf{\small($\pm$0.0018)} & {\small($\pm$0.0014)} & \textbf{\small($\pm$0.0020)} & \textbf{\small($\pm$0.0021)} \\
        \addlinespace
        superconductivity & 0.9028 & \textbf{0.9003} & \textbf{0.9010} & 0.9053 & 0.8957 & 0.9194 & \textbf{0.8999} & \textbf{0.9000} \\
                 & {\small($\pm$0.0018)} & \textbf{\small($\pm$0.0019)} & \textbf{\small($\pm$0.0019)} & {\small($\pm$0.0021)} & {\small($\pm$0.0019)} & {\small($\pm$0.0012)} & \textbf{\small($\pm$0.0016)} & \textbf{\small($\pm$0.0019)} \\
        \addlinespace
        WEC & 0.9154 & \textbf{0.9005} & \textbf{0.8999} & 0.9045 & 0.9020 & 0.9487 & \textbf{0.9002} & \textbf{0.8989} \\
                 & {\small($\pm$0.0012)} & \textbf{\small($\pm$0.0011)} & \textbf{\small($\pm$0.0011)} & {\small($\pm$0.0011)} & {\small($\pm$0.0010)} & {\small($\pm$0.0008)} & \textbf{\small($\pm$0.0009)} & \textbf{\small($\pm$0.0010)} \\
        \addlinespace
        kernel & \textbf{0.9003} & \textbf{0.8999} & \textbf{0.8998} & \textbf{0.8999} & 0.8991 & 0.9319 & \textbf{0.8998} & \textbf{0.8998} \\
                 & \textbf{\small($\pm$0.0005)} & \textbf{\small($\pm$0.0006)} & \textbf{\small($\pm$0.0005)} & \textbf{\small($\pm$0.0006)} & {\small($\pm$0.0005)} & {\small($\pm$0.0004)} & \textbf{\small($\pm$0.0005)} & \textbf{\small($\pm$0.0005)} \\
        \addlinespace
        \bottomrule
    \end{tabular}
    \end{adjustbox}
\end{table}

% Table~\ref{tab:length} evaluates the efficiency of the valid prediction intervals by comparing their average lengths across models. \ourmethod{} consistently yields tight, well-calibrated prediction intervals across diverse domain scales. On complex and larger-scale datasets such as \texttt{bike} and \texttt{WEC}, uncalibrated or rescaled parametric baselines like NGBoost produce substantially wider intervals (e.g., $124588.28$ on \texttt{WEC} for standard NGB compared to $43356.95$ for \ourmethod{}). While weighted conformal methods such as WICP achieve the shortest interval lengths in several benchmark settings, \ourmethod{} provides highly competitive sharpness while avoiding the heavy computational overhead associated with full localized density estimation at test time.

% \input{tables/my_results_length_GradBoost.tex}

Table~\ref{tab:time} presents the post-fit runtime efficiency across all methods. While standard ICP achieves the lowest absolute runtime due to its global threshold calculation, \ourmethod{} demonstrates significant speedups over localized non-parametric conformal baselines (WICP and MICP). On computationally intensive datasets such as \texttt{superconductivity}, \texttt{WEC}, and \texttt{kernel}, WICP incurs extreme post-fit latencies ($80.18$s, $179.97$s, and $253.11$s, respectively) due to test-time localized density estimations. In contrast, \ourmethod{} completes calibration and inference in $0.45$s, $2.44$s, and $12.28$s, offering orders-of-magnitude faster processing while maintaining spatially adaptive interval construction.

\begin{table}[htb]
    \centering
    \caption{\textbf{Post-Fit Computational Execution Time.} Average execution time (in seconds) required for post-fit calibration and interval generation across datasets. Values represent mean $\pm$ standard error at a 95\% confidence level over matched random splits. Bold entries denote the two lowest runtimes per dataset. ICP consistently achieves the lowest absolute runtime due to its use of a single global threshold; therefore, highlighting the second-lowest runtime also facilitates comparisons beyond ICP.}
    \label{tab:time}
    \begin{adjustbox}{max width=\columnwidth}
    \begin{tabular}{lcccccccc}
        \toprule
        \bfseries Dataset & LoBoost & ICP & WICP & MICP & LoCART & NGB & NGB-Rescale & NGB-Tuned \\
        \midrule
        airfoil & 0.0275 & \textbf{0.0049} & 0.0797 & 0.1944 & \textbf{0.0074} & 0.0500 & 0.1026 & 0.0653 \\
                 & {\small($\pm$0.0035)} & \textbf{\small($\pm$0.0006)} & {\small($\pm$0.0120)} & {\small($\pm$0.0093)} & \textbf{\small($\pm$0.0006)} & {\small($\pm$0.0041)} & {\small($\pm$0.0087)} & {\small($\pm$0.0044)} \\
        \addlinespace
        winered & 0.0140 & \textbf{0.0027} & 0.1360 & 0.3496 & \textbf{0.0052} & 0.0478 & 0.1012 & 0.0701 \\
                 & {\small($\pm$0.0029)} & \textbf{\small($\pm$0.0005)} & {\small($\pm$0.0255)} & {\small($\pm$0.0126)} & \textbf{\small($\pm$0.0005)} & {\small($\pm$0.0035)} & {\small($\pm$0.0072)} & {\small($\pm$0.0054)} \\
        \addlinespace
        star & 0.0118 & \textbf{0.0023} & 0.0882 & 0.8935 & \textbf{0.0054} & 0.0620 & 0.1386 & 0.1038 \\
                 & {\small($\pm$0.0024)} & \textbf{\small($\pm$0.0005)} & {\small($\pm$0.0254)} & {\small($\pm$0.0324)} & \textbf{\small($\pm$0.0005)} & {\small($\pm$0.0041)} & {\small($\pm$0.0084)} & {\small($\pm$0.0079)} \\
        \addlinespace
        winewhite & 0.0568 & \textbf{0.0134} & 0.7474 & 1.2011 & \textbf{0.0206} & 0.0654 & 0.1462 & 0.1073 \\
                 & {\small($\pm$0.0091)} & \textbf{\small($\pm$0.0021)} & {\small($\pm$0.1489)} & {\small($\pm$0.0347)} & \textbf{\small($\pm$0.0021)} & {\small($\pm$0.0034)} & {\small($\pm$0.0082)} & {\small($\pm$0.0066)} \\
        \addlinespace
        cycle & 0.1641 & \textbf{0.0367} & 1.2768 & 1.4764 & \textbf{0.0504} & 0.0634 & 0.1436 & 0.1076 \\
                 & {\small($\pm$0.0192)} & \textbf{\small($\pm$0.0042)} & {\small($\pm$0.1828)} & {\small($\pm$0.0533)} & \textbf{\small($\pm$0.0043)} & {\small($\pm$0.0038)} & {\small($\pm$0.0085)} & {\small($\pm$0.0070)} \\
        \addlinespace
        electric & \textbf{0.0285} & \textbf{0.0055} & 0.4582 & 3.6982 & 0.0392 & 0.0838 & 0.1975 & 0.1591 \\
                 & \textbf{\small($\pm$0.0006)} & \textbf{\small($\pm$0.0001)} & {\small($\pm$0.0135)} & {\small($\pm$0.1325)} & {\small($\pm$0.0014)} & {\small($\pm$0.0040)} & {\small($\pm$0.0086)} & {\small($\pm$0.0079)} \\
        \addlinespace
        bike & 0.2170 & \textbf{0.0554} & 1.6694 & 2.3576 & \textbf{0.0845} & 0.1051 & 0.2513 & 0.1774 \\
                 & {\small($\pm$0.0160)} & \textbf{\small($\pm$0.0041)} & {\small($\pm$0.1952)} & {\small($\pm$0.0936)} & \textbf{\small($\pm$0.0046)} & {\small($\pm$0.0058)} & {\small($\pm$0.0129)} & {\small($\pm$0.0119)} \\
        \addlinespace
        meps19 & \textbf{0.1554} & \textbf{0.0348} & 3.1145 & 14.9875 & 0.3038 & 0.9958 & 4.4649 & 2.1043 \\
                 & \textbf{\small($\pm$0.0195)} & \textbf{\small($\pm$0.0053)} & {\small($\pm$0.6965)} & {\small($\pm$0.8288)} & {\small($\pm$0.0175)} & {\small($\pm$0.2744)} & {\small($\pm$0.7889)} & {\small($\pm$0.5149)} \\
        \addlinespace
        superconductivity & 0.4473 & \textbf{0.0905} & 80.1843 & 45.2347 & 0.8566 & \textbf{0.3481} & 0.9326 & 0.7259 \\
                 & {\small($\pm$0.0405)} & \textbf{\small($\pm$0.0091)} & {\small($\pm$10.4591)} & {\small($\pm$1.5367)} & {\small($\pm$0.0361)} & \textbf{\small($\pm$0.0142)} & {\small($\pm$0.0472)} & {\small($\pm$0.0323)} \\
        \addlinespace
        WEC & \textbf{2.4396} & \textbf{0.4487} & 179.9673 & 56.1652 & 3.4989 & 2.6285 & 8.2697 & 4.1058 \\
                 & \textbf{\small($\pm$0.0884)} & \textbf{\small($\pm$0.0224)} & {\small($\pm$16.9139)} & {\small($\pm$2.1785)} & {\small($\pm$0.1584)} & {\small($\pm$0.2685)} & {\small($\pm$0.7862)} & {\small($\pm$0.5368)} \\
        \addlinespace
        kernel & 12.2817 & \textbf{3.7622} & 253.1148 & 91.6793 & 12.4294 & \textbf{5.1647} & 14.7159 & 7.4983 \\
                 & {\small($\pm$1.3824)} & \textbf{\small($\pm$0.4880)} & {\small($\pm$44.4115)} & {\small($\pm$8.4255)} & {\small($\pm$0.7231)} & \textbf{\small($\pm$0.7348)} & {\small($\pm$2.0229)} & {\small($\pm$1.6163)} \\
        \addlinespace
        \bottomrule
    \end{tabular}
    \end{adjustbox}
\end{table}

Table~\ref{tab:smis} evaluates overall prediction interval quality using the SMIS. As formulated in Equation~\ref{eq:smis}, lower values represent superior performance, effectively rewarding narrow intervals while penalizing miscoverage proportional to $2/\alpha$. \ourmethod{} consistently achieves highly competitive or superior SMIS metrics across diverse benchmark settings. Standard NGBoost suffers from substantially degraded SMIS scores on datasets like \texttt{winered} and \texttt{star} due to heavy penalties triggered by coverage underestimation. Conversely, \ourmethod{} delivers well-calibrated intervals that achieve optimal score quality, matching or surpassing computationally intensive localized baselines while maintaining low post-fit computational complexity.

\begin{table}[htb]
    \centering
    \caption{\textbf{Standard Mean Interval Score Results.} Assessment of prediction interval quality at the $\alpha = 0.10$ miscoverage level. Lower values indicate superior performance by jointly penalizing interval width and coverage departures. Values represent the mean $\pm$ standard error at a 95\% confidence level over matched random splits. Bold entries denote the best (lowest) score per dataset.}
    \label{tab:smis}
    \begin{adjustbox}{max width=\columnwidth}
    \begin{tabular}{lcccccccc}
        \toprule
        \bfseries Dataset & \ourmethod{} & ICP & WICP & MICP & LoCART & NGB & NGB-Rescale & NGB-Tuned \\
        \midrule
        airfoil & 11.0787 & 11.1516 & 9.5658 & 11.2129 & 10.8680 & 11.2599 & 11.2690 & \textbf{9.1082} \\
                 & {\small($\pm$0.3058)} & {\small($\pm$0.3013)} & {\small($\pm$0.2106)} & {\small($\pm$0.2991)} & {\small($\pm$0.2837)} & {\small($\pm$0.2313)} & {\small($\pm$0.2226)} & \textbf{\small($\pm$0.1938)} \\
        \addlinespace
        winered & \textbf{2.7714} & \textbf{2.7750} & \textbf{2.7769} & \textbf{2.8321} & \textbf{2.7606} & 2.8980 & \textbf{2.8037} & \textbf{2.7896} \\
                 & \textbf{\small($\pm$0.0531)} & \textbf{\small($\pm$0.0529)} & \textbf{\small($\pm$0.0566)} & \textbf{\small($\pm$0.0473)} & \textbf{\small($\pm$0.0557)} & {\small($\pm$0.0746)} & \textbf{\small($\pm$0.0513)} & \textbf{\small($\pm$0.0613)} \\
        \addlinespace
        star & \textbf{953.6340} & \textbf{948.3424} & \textbf{954.8437} & 990.1628 & \textbf{954.3126} & 978.9640 & \textbf{955.1889} & \textbf{960.3432} \\
                 & \textbf{\small($\pm$10.6277)} & \textbf{\small($\pm$10.8504)} & \textbf{\small($\pm$12.1971)} & {\small($\pm$9.5289)} & \textbf{\small($\pm$10.7606)} & {\small($\pm$11.5101)} & \textbf{\small($\pm$8.7578)} & \textbf{\small($\pm$11.1007)} \\
        \addlinespace
        winewhite & 2.9386 & 2.9384 & 2.9268 & \textbf{2.8043} & 2.9222 & 2.9352 & 2.9227 & 2.9769 \\
                 & {\small($\pm$0.0349)} & {\small($\pm$0.0348)} & {\small($\pm$0.0388)} & \textbf{\small($\pm$0.0309)} & {\small($\pm$0.0343)} & {\small($\pm$0.0417)} & {\small($\pm$0.0368)} & {\small($\pm$0.0377)} \\
        \addlinespace
        cycle & 14.5443 & 14.5064 & \textbf{14.1495} & \textbf{14.0937} & 14.5150 & 15.8181 & 15.8155 & 14.6891 \\
                 & {\small($\pm$0.1842)} & {\small($\pm$0.1891)} & \textbf{\small($\pm$0.1836)} & \textbf{\small($\pm$0.1812)} & {\small($\pm$0.1871)} & {\small($\pm$0.1729)} & {\small($\pm$0.1695)} & {\small($\pm$0.1756)} \\
        \addlinespace
        electric & 0.0497 & 0.0500 & 0.0446 & 0.0486 & 0.0482 & 0.0563 & 0.0557 & \textbf{0.0368} \\
                 & {\small($\pm$0.0004)} & {\small($\pm$0.0004)} & {\small($\pm$0.0003)} & {\small($\pm$0.0004)} & {\small($\pm$0.0004)} & {\small($\pm$0.0004)} & {\small($\pm$0.0004)} & \textbf{\small($\pm$0.0003)} \\
        \addlinespace
        bike & 183.9175 & 190.3261 & \textbf{140.7535} & 150.8943 & 160.2504 & 254.8795 & 254.9643 & 226.4414 \\
                 & {\small($\pm$2.4072)} & {\small($\pm$2.4990)} & \textbf{\small($\pm$1.6935)} & {\small($\pm$1.7020)} & {\small($\pm$2.2572)} & {\small($\pm$1.7630)} & {\small($\pm$2.3001)} & {\small($\pm$2.9551)} \\
        \addlinespace
        meps19 & 73.6486 & 96.1073 & \textbf{66.8293} & \textbf{69.8958} & \textbf{70.6251} & \textbf{69.3255} & \textbf{69.8249} & \textbf{67.5690} \\
                 & {\small($\pm$2.7092)} & {\small($\pm$2.4244)} & \textbf{\small($\pm$2.1653)} & \textbf{\small($\pm$1.9774)} & \textbf{\small($\pm$1.8190)} & \textbf{\small($\pm$1.9447)} & \textbf{\small($\pm$2.1466)} & \textbf{\small($\pm$2.0106)} \\
        \addlinespace
        superconductivity & 46.2242 & 50.0430 & \textbf{38.2761} & \textbf{37.7812} & 40.2664 & 47.4750 & 47.6193 & 41.1471 \\
                 & {\small($\pm$0.5104)} & {\small($\pm$0.4471)} & \textbf{\small($\pm$0.5238)} & \textbf{\small($\pm$0.3916)} & {\small($\pm$0.4355)} & {\small($\pm$0.3446)} & {\small($\pm$0.3634)} & {\small($\pm$0.3503)} \\
        \addlinespace
        WEC & 78182.1068 & 104540.9878 & 66608.2201 & \textbf{62203.2665} & 74892.1365 & 158833.9667 & 157943.3760 & 103994.2840 \\
                 & {\small($\pm$921.7007)} & {\small($\pm$824.2052)} & {\small($\pm$659.3666)} & \textbf{\small($\pm$598.1897)} & {\small($\pm$671.7611)} & {\small($\pm$600.1490)} & {\small($\pm$717.9266)} & {\small($\pm$1106.8499)} \\
        \addlinespace
        kernel & 9.5733 & 10.3061 & \textbf{5.6754} & 6.6869 & 6.3462 & 9.9007 & 9.9981 & 10.0254 \\
                 & {\small($\pm$0.0538)} & {\small($\pm$0.0544)} & \textbf{\small($\pm$0.0262)} & {\small($\pm$0.0327)} & {\small($\pm$0.0266)} & {\small($\pm$0.0443)} & {\small($\pm$0.0516)} & {\small($\pm$0.0512)} \\
        \addlinespace
        \bottomrule
    \end{tabular}
    \end{adjustbox}
\end{table}

\clearpage

% #### New reviewer-requested results begin here. Tables and figures remain black for readability.
\section{Additional Reviewer Results}
\label{appendix:reviewer_results}

\subsection{All-method local coverage and direct LoBoost--LoCART comparison}
We extend the WSC diagnostic to all shared-predictor conformal methods, using the same 1,000 random directions in each matched split. \ourmethod has the largest mean WSC on winered, star, superconductivity, and WEC, whereas WICP, MICP, and LoCART are best on other datasets. Thus local coverage is not a \ourmethod-only advantage. The direct \ourmethod--LoCART tables report the requested paired SMIS intervals, WSC, post-fit time, and total time.
% #### Added for reviewer response; table text remains black for readability.
\begin{table}[htb]
    \centering
    \caption{\textbf{Worst-Slab Coverage Results.} Worst-slab coverage (WSC) for the shared-predictor conformal methods over 50 matched splits. Higher values indicate better worst-case local coverage. Values represent the mean $\pm$ standard error over matched random splits. Bold entries denote the best (highest) WSC per dataset.}
    \label{tab:reviewer_wsc_all_methods}
    \begin{adjustbox}{max width=\columnwidth}
    \begin{tabular}{lccccc}
        \toprule
        \bfseries Dataset & \ourmethod{} & ICP & WICP & MICP & LoCART \\
        \midrule

        airfoil & 0.714 & 0.694 & \textbf{0.726} & 0.725 & 0.710 \\
                & {\small($\pm$0.008)} & {\small($\pm$0.008)} & \textbf{\small($\pm$0.007)} & {\small($\pm$0.006)} & {\small($\pm$0.007)} \\
        \addlinespace

        winered & \textbf{0.725} & 0.714 & 0.715 & 0.717 & 0.717 \\
                & \textbf{\small($\pm$0.007)} & {\small($\pm$0.007)} & {\small($\pm$0.009)} & {\small($\pm$0.007)} & {\small($\pm$0.007)} \\
        \addlinespace

        star & \textbf{0.748} & 0.741 & 0.740 & \textbf{0.748} & 0.738 \\
             & \textbf{\small($\pm$0.005)} & {\small($\pm$0.005)} & {\small($\pm$0.005)} & \textbf{\small($\pm$0.005)} & {\small($\pm$0.005)} \\
        \addlinespace

        winewhite & 0.807 & 0.782 & 0.789 & \textbf{0.813} & 0.787 \\
                  & {\small($\pm$0.003)} & {\small($\pm$0.003)} & {\small($\pm$0.003)} & \textbf{\small($\pm$0.003)} & {\small($\pm$0.003)} \\
        \addlinespace

        cycle & 0.830 & 0.821 & 0.825 & \textbf{0.837} & 0.818 \\
              & {\small($\pm$0.002)} & {\small($\pm$0.003)} & {\small($\pm$0.003)} & \textbf{\small($\pm$0.002)} & {\small($\pm$0.003)} \\
        \addlinespace

        electric & 0.809 & 0.773 & \textbf{0.823} & 0.804 & 0.798 \\
                 & {\small($\pm$0.003)} & {\small($\pm$0.003)} & \textbf{\small($\pm$0.002)} & {\small($\pm$0.003)} & {\small($\pm$0.003)} \\
        \addlinespace

        bike & 0.795 & 0.773 & \textbf{0.830} & 0.811 & 0.804 \\
             & {\small($\pm$0.003)} & {\small($\pm$0.003)} & \textbf{\small($\pm$0.002)} & {\small($\pm$0.003)} & {\small($\pm$0.004)} \\
        \addlinespace

        meps19 & 0.827 & 0.657 & 0.816 & \textbf{0.838} & 0.813 \\
               & {\small($\pm$0.002)} & {\small($\pm$0.003)} & {\small($\pm$0.003)} & \textbf{\small($\pm$0.002)} & {\small($\pm$0.003)} \\
        \addlinespace

        superconductivity & \textbf{0.845} & 0.747 & 0.796 & 0.800 & 0.823 \\
                          & \textbf{\small($\pm$0.001)} & {\small($\pm$0.002)} & {\small($\pm$0.002)} & {\small($\pm$0.002)} & {\small($\pm$0.002)} \\
        \addlinespace

        WEC & \textbf{0.860} & 0.770 & 0.813 & 0.807 & 0.853 \\
            & \textbf{\small($\pm$0.001)} & {\small($\pm$0.002)} & {\small($\pm$0.002)} & {\small($\pm$0.002)} & {\small($\pm$0.001)} \\
        \addlinespace

        kernel & 0.855 & 0.725 & 0.876 & 0.843 & \textbf{0.881} \\
               & {\small($\pm$0.00063)} & {\small($\pm$0.002)} & {\small($\pm$0.00056)} & {\small($\pm$0.00085)} & \textbf{\small($\pm$0.00042)} \\
        \addlinespace

        \bottomrule
    \end{tabular}
    \end{adjustbox}
\end{table}

\begin{table*}[t]
\centering
\small
\caption{Direct LoBoost--LoCART comparison over 50 matched splits. $\Delta\mathrm{SMIS}=\mathrm{SMIS}(\mathrm{LoBoost})-\mathrm{SMIS}(\mathrm{LoCART})$, so positive values favor LoCART. WSC is reported as LoBoost/LoCART; $^*$ denotes a paired 95\% interval containing zero.}
\label{tab:reviewer_loboost_locart_quality}
\begin{adjustbox}{max width=\textwidth}
\begin{tabular}{lcccc}
\toprule
Dataset & LoBoost SMIS & LoCART SMIS & $\Delta$SMIS [95\% CI] & WSC (LB/LC) \\
\midrule
airfoil (1503, 5) & 10.932 & 10.868 & 0.064 [-0.069, 0.198]$^*$ & 0.714/0.710 \\
winered (1599, 11) & 2.776 & 2.761 & 0.015 [-0.007, 0.038]$^*$ & 0.725/0.717 \\
star (2161, 39) & 956.462 & 954.313 & 2.149 [-2.240, 6.539]$^*$ & 0.748/0.738 \\
winewhite (4898, 11) & 2.920 & 2.922 & -0.002 [-0.013, 0.009]$^*$ & 0.807/0.787 \\
cycle (9568, 4) & 14.601 & 14.515 & 0.086 [0.049, 0.123] & 0.830/0.818 \\
electric (10000, 12) & 0.049 & 0.048 & 6.05e-04 [3.12e-04, 8.97e-04] & 0.809/0.798 \\
bike (10886, 18) & 161.240 & 160.250 & 0.989 [-0.400, 2.379]$^*$ & 0.795/0.804 \\
meps19 (15785, 139) & 71.447 & 70.625 & 0.822 [0.072, 1.572] & 0.827/0.813 \\
conductivity (21263, 81) & 40.327 & 40.266 & 0.061 [-0.237, 0.359]$^*$ & 0.845/0.823 \\
WEC (54007, 98) & 8.00e+04 & 7.49e+04 & 5133.233 [4637.131, 5629.335] & 0.860/0.853 \\
kernel (241600, 15) & 7.466 & 6.346 & 1.120 [1.090, 1.150] & 0.855/0.881 \\
\bottomrule
\end{tabular}
\end{adjustbox}
\end{table*}

% #### Added for reviewer response; table text remains black for readability.
\begin{table*}[t]
\centering
\small
\caption{LoBoost and LoCART computational cost over 50 matched splits. Post-fit time covers calibration and interval generation; total time includes model fitting.}
\label{tab:reviewer_loboost_locart_time}
\begin{adjustbox}{max width=\textwidth}
\begin{tabular}{lcccc}
\toprule
Dataset & LB post-fit (s) & LC post-fit (s) & LB total (s) & LC total (s) \\
\midrule
airfoil (1503, 5) & 0.041 & 0.007 & 31.2 & 31.1 \\
winered (1599, 11) & 0.020 & 0.005 & 19.9 & 19.9 \\
star (2161, 39) & 0.019 & 0.005 & 33.1 & 33.1 \\
winewhite (4898, 11) & 0.099 & 0.021 & 103.3 & 103.2 \\
cycle (9568, 4) & 0.292 & 0.050 & 309.8 & 309.5 \\
electric (10000, 12) & 0.071 & 0.039 & 79.3 & 79.2 \\
bike (10886, 18) & 0.374 & 0.084 & 417.7 & 417.4 \\
meps19 (15785, 139) & 0.428 & 0.304 & 524.4 & 524.2 \\
conductivity (21263, 81) & 1.561 & 0.857 & 1.46e+04 & 1.46e+04 \\
WEC (54007, 98) & 13.805 & 3.499 & 2.67e+04 & 2.67e+04 \\
kernel (241600, 15) & 25.768 & 12.429 & 1.28e+04 & 1.28e+04 \\
\bottomrule
\end{tabular}
\end{adjustbox}
\end{table*}

\subsection{Quality--runtime Pareto view and complete benchmark metrics}
Figure~\ref{fig:reviewer_pareto} compares geometric-mean relative SMIS with geometric-mean relative post-fit and total runtime. The normalization is within dataset, so response-scale differences do not dominate the aggregate. \ourmethod is not on either aggregate Pareto frontier; this confirms that its role is an intermediate quality--cost trade-off rather than uniform Pareto dominance.
\begin{figure*}[t]
  \centering
  \includegraphics[width=0.90\textwidth]{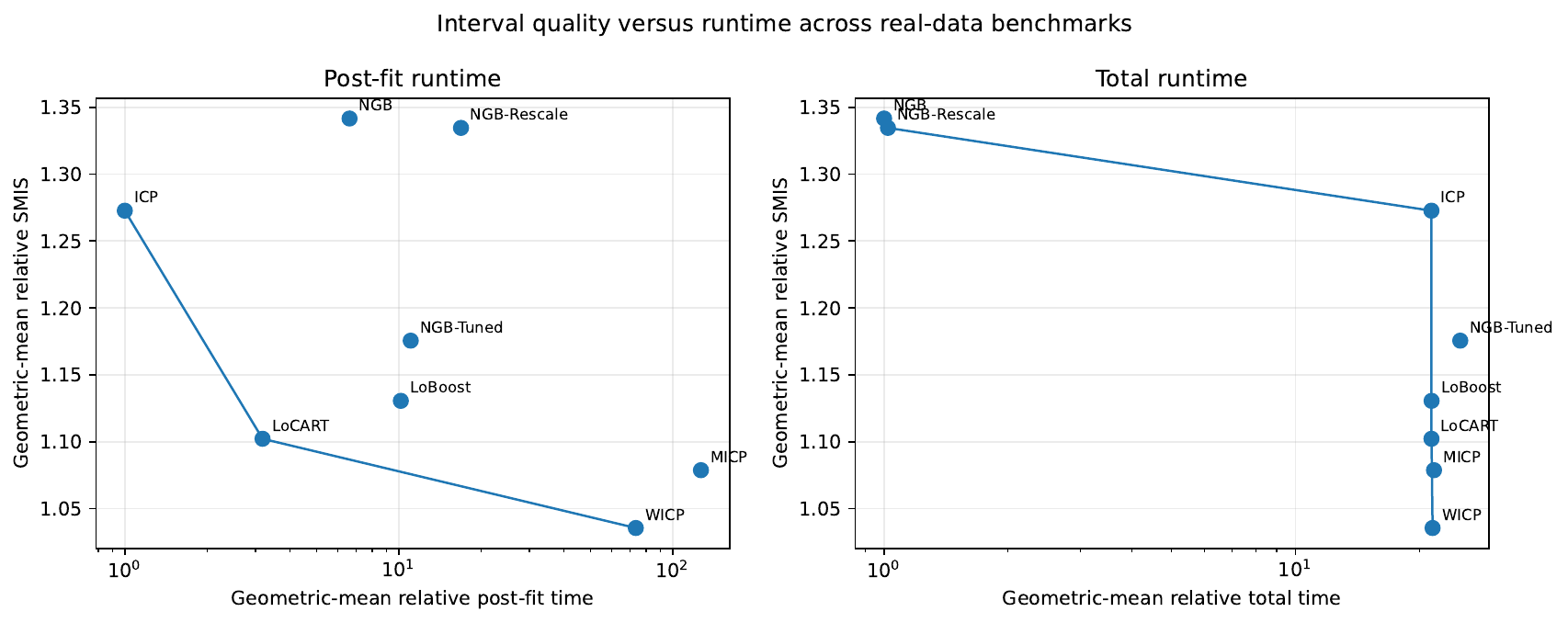}
  \caption{Aggregate Pareto-style comparison of relative SMIS (lower is better) with post-fit and total runtime (lower is better). Values are geometric means of ratios to the best method within each dataset.}
  \label{fig:reviewer_pareto}
\end{figure*}
For completeness, Tables~\ref{tab:reviewer_length} and \ref{tab:reviewer_mse} provide the interval-length and test-MSE results announced in the experiments section. The shared-predictor conformal methods have identical MSE by construction; their interval differences are therefore attributable to calibration.
% #### Added for reviewer response; table text remains black for readability.
\begin{table}[!h]
    \centering
    \caption{Mean interval length across 50 matched splits. Values are mean $\pm$ 95\% confidence-interval half-width.}
    \label{tab:reviewer_length}
    \begin{adjustbox}{max width=\columnwidth}
    \begin{tabular}{lcccccccc}
        \toprule
        \bfseries Dataset & LoBoost & ICP & WICP & MICP & LoCART & NGB & NGB-Rescale & NGB-Tuned \\
        \midrule
        airfoil & 7.9629 & 7.7263 & 7.4712 & 8.6511 & 7.7131 & 8.5624 & 8.9780 & \textbf{7.1234} \\
                 & {\small($\pm$0.1788)} & {\small($\pm$0.1768)} & {\small($\pm$0.1675)} & {\small($\pm$0.2131)} & {\small($\pm$0.1821)} & {\small($\pm$0.0739)} & {\small($\pm$0.1244)} & \textbf{\small($\pm$0.1311)} \\
        \addlinespace
        winered & 2.0992 & 2.0456 & 2.0779 & 2.1856 & 2.0761 & \textbf{1.5681} & 2.0259 & 2.0631 \\
                 & {\small($\pm$0.0331)} & {\small($\pm$0.0308)} & {\small($\pm$0.0407)} & {\small($\pm$0.0395)} & {\small($\pm$0.0316)} & \textbf{\small($\pm$0.0116)} & {\small($\pm$0.0402)} & {\small($\pm$0.0350)} \\
        \addlinespace
        star & 789.6777 & 774.8195 & 777.1504 & 808.4332 & 775.1885 & \textbf{662.4284} & 774.3464 & 781.3825 \\
                 & {\small($\pm$6.1600)} & {\small($\pm$5.5942)} & {\small($\pm$6.1261)} & {\small($\pm$6.5469)} & {\small($\pm$6.1821)} & \textbf{\small($\pm$2.6011)} & {\small($\pm$7.2952)} & {\small($\pm$6.9652)} \\
        \addlinespace
        winewhite & 2.2474 & 2.1807 & 2.2056 & 2.2100 & 2.1755 & \textbf{2.0740} & 2.2171 & 2.2090 \\
                 & {\small($\pm$0.0214)} & {\small($\pm$0.0195)} & {\small($\pm$0.0270)} & {\small($\pm$0.0189)} & {\small($\pm$0.0208)} & \textbf{\small($\pm$0.0051)} & {\small($\pm$0.0177)} & {\small($\pm$0.0227)} \\
        \addlinespace
        cycle & 10.7454 & 10.4228 & \textbf{10.2309} & 10.4316 & 10.4125 & 11.6687 & 11.8135 & 10.9064 \\
                 & {\small($\pm$0.0708)} & {\small($\pm$0.0720)} & \textbf{\small($\pm$0.0817)} & {\small($\pm$0.0684)} & {\small($\pm$0.0637)} & {\small($\pm$0.0306)} & {\small($\pm$0.0496)} & {\small($\pm$0.0559)} \\
        \addlinespace
        electric & 0.0375 & 0.0368 & 0.0347 & 0.0374 & 0.0362 & 0.0482 & 0.0438 & \textbf{0.0298} \\
                 & {\small($\pm$0.0002)} & {\small($\pm$0.0002)} & {\small($\pm$0.0002)} & {\small($\pm$0.0002)} & {\small($\pm$0.0002)} & {\small($\pm$0.0003)} & {\small($\pm$0.0003)} & \textbf{\small($\pm$0.0002)} \\
        \addlinespace
        bike & 111.2293 & 116.5122 & \textbf{106.2511} & 108.7733 & \textbf{108.4223} & 223.6734 & 195.5019 & 179.9057 \\
                 & {\small($\pm$1.0487)} & {\small($\pm$1.0225)} & \textbf{\small($\pm$1.6653)} & {\small($\pm$0.8479)} & \textbf{\small($\pm$0.9463)} & {\small($\pm$0.9307)} & {\small($\pm$1.2239)} & {\small($\pm$2.0161)} \\
        \addlinespace
        meps19 & 32.4831 & 33.2444 & \textbf{23.5635} & 30.7683 & 31.5177 & 33.3383 & \textbf{23.9410} & \textbf{25.2603} \\
                 & {\small($\pm$0.7737)} & {\small($\pm$0.5934)} & \textbf{\small($\pm$1.2390)} & {\small($\pm$0.6305)} & {\small($\pm$0.6849)} & {\small($\pm$0.4382)} & \textbf{\small($\pm$0.3116)} & \textbf{\small($\pm$0.5118)} \\
        \addlinespace
        superconductivity & 27.7881 & 29.8955 & \textbf{23.9471} & 24.9402 & 26.4708 & 36.4740 & 34.2376 & 28.8510 \\
                 & {\small($\pm$0.1906)} & {\small($\pm$0.2212)} & \textbf{\small($\pm$0.2855)} & {\small($\pm$0.1745)} & {\small($\pm$0.1916)} & {\small($\pm$0.1203)} & {\small($\pm$0.1395)} & {\small($\pm$0.1390)} \\
        \addlinespace
        WEC & 43668.8749 & 45384.0083 & \textbf{29259.0831} & 37863.9899 & 42127.8627 & 124588.2804 & 99863.8306 & 64708.9189 \\
                 & {\small($\pm$444.9728)} & {\small($\pm$435.5355)} & \textbf{\small($\pm$323.8928)} & {\small($\pm$370.4982)} & {\small($\pm$428.7716)} & {\small($\pm$365.6146)} & {\small($\pm$353.7681)} & {\small($\pm$619.5814)} \\
        \addlinespace
        kernel & 4.5459 & 4.2896 & \textbf{3.5730} & 4.3030 & 4.2071 & 6.7153 & 5.4689 & 5.4825 \\
                 & {\small($\pm$0.0196)} & {\small($\pm$0.0118)} & \textbf{\small($\pm$0.0120)} & {\small($\pm$0.0140)} & {\small($\pm$0.0140)} & {\small($\pm$0.0152)} & {\small($\pm$0.0125)} & {\small($\pm$0.0138)} \\
        \addlinespace
        \bottomrule
    \end{tabular}
    \end{adjustbox}
\end{table}
% #### Added for reviewer response; table text remains black for readability.
\begin{table}[!h]
    \centering
    \caption{Test MSE across 50 matched splits. Values are mean $\pm$ 95\% confidence-interval half-width.}
    \label{tab:reviewer_mse}
    \begin{adjustbox}{max width=\columnwidth}
    \begin{tabular}{lcccccccc}
        \toprule
        \bfseries Dataset & LoBoost & ICP & WICP & MICP & LoCART & NGB & NGB-Rescale & NGB-Tuned \\
        \midrule
        airfoil & \textbf{5.8190} & \textbf{5.8190} & \textbf{5.8190} & \textbf{5.8190} & \textbf{5.8190} & 9.3143 & 9.3143 & \textbf{5.6306} \\
                 & \textbf{\small($\pm$0.2806)} & \textbf{\small($\pm$0.2806)} & \textbf{\small($\pm$0.2806)} & \textbf{\small($\pm$0.2806)} & \textbf{\small($\pm$0.2806)} & {\small($\pm$0.3078)} & {\small($\pm$0.3078)} & \textbf{\small($\pm$0.2099)} \\
        \addlinespace
        winered & \textbf{0.3948} & \textbf{0.3948} & \textbf{0.3948} & \textbf{0.3948} & \textbf{0.3948} & \textbf{0.3991} & \textbf{0.3991} & \textbf{0.4019} \\
                 & \textbf{\small($\pm$0.0122)} & \textbf{\small($\pm$0.0122)} & \textbf{\small($\pm$0.0122)} & \textbf{\small($\pm$0.0122)} & \textbf{\small($\pm$0.0122)} & \textbf{\small($\pm$0.0113)} & \textbf{\small($\pm$0.0113)} & \textbf{\small($\pm$0.0115)} \\
        \addlinespace
        star & \textbf{55770.5188} & \textbf{55770.5188} & \textbf{55770.5188} & \textbf{55770.5188} & \textbf{55770.5188} & \textbf{54909.6153} & \textbf{54909.6153} & \textbf{54960.6736} \\
                 & \textbf{\small($\pm$946.3835)} & \textbf{\small($\pm$946.3835)} & \textbf{\small($\pm$946.3835)} & \textbf{\small($\pm$946.3835)} & \textbf{\small($\pm$946.3835)} & \textbf{\small($\pm$889.0664)} & \textbf{\small($\pm$889.0664)} & \textbf{\small($\pm$953.8658)} \\
        \addlinespace
        winewhite & \textbf{0.4522} & \textbf{0.4522} & \textbf{0.4522} & \textbf{0.4522} & \textbf{0.4522} & 0.4943 & 0.4943 & \textbf{0.4684} \\
                 & \textbf{\small($\pm$0.0089)} & \textbf{\small($\pm$0.0089)} & \textbf{\small($\pm$0.0089)} & \textbf{\small($\pm$0.0089)} & \textbf{\small($\pm$0.0089)} & {\small($\pm$0.0087)} & {\small($\pm$0.0087)} & \textbf{\small($\pm$0.0091)} \\
        \addlinespace
        cycle & \textbf{11.3378} & \textbf{11.3378} & \textbf{11.3378} & \textbf{11.3378} & \textbf{11.3378} & 15.3078 & 15.3078 & 13.1070 \\
                 & \textbf{\small($\pm$0.3148)} & \textbf{\small($\pm$0.3148)} & \textbf{\small($\pm$0.3148)} & \textbf{\small($\pm$0.3148)} & \textbf{\small($\pm$0.3148)} & {\small($\pm$0.2834)} & {\small($\pm$0.2834)} & {\small($\pm$0.2875)} \\
        \addlinespace
        electric & 0.0001 & 0.0001 & 0.0001 & 0.0001 & 0.0001 & 0.0002 & 0.0002 & \textbf{0.0001} \\
                 & {\small($\pm$0.0000)} & {\small($\pm$0.0000)} & {\small($\pm$0.0000)} & {\small($\pm$0.0000)} & {\small($\pm$0.0000)} & {\small($\pm$0.0000)} & {\small($\pm$0.0000)} & \textbf{\small($\pm$0.0000)} \\
        \addlinespace
        bike & \textbf{1532.1736} & \textbf{1532.1736} & \textbf{1532.1736} & \textbf{1532.1736} & \textbf{1532.1736} & 6959.1337 & 6959.1337 & 5484.2245 \\
                 & \textbf{\small($\pm$35.5043)} & \textbf{\small($\pm$35.5043)} & \textbf{\small($\pm$35.5043)} & \textbf{\small($\pm$35.5043)} & \textbf{\small($\pm$35.5043)} & {\small($\pm$115.2717)} & {\small($\pm$115.2717)} & {\small($\pm$172.0137)} \\
        \addlinespace
        meps19 & \textbf{591.3184} & \textbf{591.3184} & \textbf{591.3184} & \textbf{591.3184} & \textbf{591.3184} & \textbf{603.5975} & \textbf{603.5975} & \textbf{589.8656} \\
                 & \textbf{\small($\pm$37.5593)} & \textbf{\small($\pm$37.5593)} & \textbf{\small($\pm$37.5593)} & \textbf{\small($\pm$37.5593)} & \textbf{\small($\pm$37.5593)} & \textbf{\small($\pm$36.7896)} & \textbf{\small($\pm$36.7896)} & \textbf{\small($\pm$36.4960)} \\
        \addlinespace
        superconductivity & \textbf{97.7929} & \textbf{97.7929} & \textbf{97.7929} & \textbf{97.7929} & \textbf{97.7929} & 171.6122 & 171.6122 & 134.9014 \\
                 & \textbf{\small($\pm$1.7252)} & \textbf{\small($\pm$1.7252)} & \textbf{\small($\pm$1.7252)} & \textbf{\small($\pm$1.7252)} & \textbf{\small($\pm$1.7252)} & {\small($\pm$1.8957)} & {\small($\pm$1.8957)} & {\small($\pm$1.6970)} \\
        \addlinespace
        WEC & \textbf{431515746.7900} & \textbf{431515746.7900} & \textbf{431515746.7900} & \textbf{431515746.7900} & \textbf{431515746.7900} & 1803670654.8178 & 1803670654.8178 & 961789742.1084 \\
                 & \textbf{\small($\pm$7564940.5225)} & \textbf{\small($\pm$7564940.5225)} & \textbf{\small($\pm$7564940.5225)} & \textbf{\small($\pm$7564940.5225)} & \textbf{\small($\pm$7564940.5225)} & {\small($\pm$15571006.6928)} & {\small($\pm$15571006.6928)} & {\small($\pm$16677378.7717)} \\
        \addlinespace
        kernel & \textbf{5.6604} & \textbf{5.6604} & \textbf{5.6604} & \textbf{5.6604} & \textbf{5.6604} & 11.6740 & 11.6740 & 11.8300 \\
                 & \textbf{\small($\pm$0.1123)} & \textbf{\small($\pm$0.1123)} & \textbf{\small($\pm$0.1123)} & \textbf{\small($\pm$0.1123)} & \textbf{\small($\pm$0.1123)} & {\small($\pm$0.1602)} & {\small($\pm$0.1602)} & {\small($\pm$0.1635)} \\
        \addlinespace
        \bottomrule
    \end{tabular}
    \end{adjustbox}
\end{table}

\subsection{Partition diagnostics and component ablations}

The diagnostics in Table \ref{tab:reviewer_partition_diagnostics} show substantial aggregation (mean merge fractions from 67.79\% to 86.23\%) while global fallback is rare (at most 1.06\% of test points). Hence the measured performance is chiefly determined by local aggregated groups rather than by fallback to the global split-conformal cutoff. The cell-size and depth diagnostics in Table \ref{tab:reviewer_cell_depth_diagnostics} show that median final cell sizes are generally between 84.3 and 107.8 observations, except for kernel, where the median is 482.1. Selected depths are also generally shallow, with median values below 3.2 for all datasets except meps19 (8.92), indicating that most selected partitions occur relatively early in the boosting sequence. The two ablations are deliberately separated: Table~\ref{tab:reviewer_weighting_ablation} changes only Hamming weighting under the fitted order, and Table~\ref{tab:reviewer_order_ablation} changes only tree order under weighted Hamming. Weighted merging significantly improves SMIS on WEC, bike, meps19, and superconductivity; randomizing tree order significantly degrades SMIS on WEC, bike, electric, meps19, superconductivity, and winewhite.

% #### Added for reviewer response; table text remains black for readability.
\begin{table*}[t]
\centering
\small
\caption{Partition diagnostics across 50 splits. The fallback fraction is the fraction of test points sent to the global split-conformal cutoff.}
\label{tab:reviewer_partition_diagnostics}
\begin{adjustbox}{max width=\textwidth}
\begin{tabular}{lcccc}
\toprule
Dataset & Raw cells & Surviving cells & Merge fraction & Fallback fraction \\
\midrule
airfoil & 13.6 & 4.2 & 67.79\% & 0.035\% \\
winered & 16.0 & 4.5 & 70.42\% & 0.117\% \\
star & 29.1 & 5.8 & 78.83\% & 0.542\% \\
winewhite & 58.7 & 13.5 & 76.14\% & 0.343\% \\
cycle & 94.9 & 26.0 & 72.47\% & 0.248\% \\
electric & 153.4 & 27.2 & 81.76\% & 1.061\% \\
bike & 103.3 & 29.5 & 71.15\% & 0.154\% \\
meps19 & 292.1 & 38.5 & 86.20\% & 0.928\% \\
superconductivity & 335.7 & 57.7 & 82.44\% & 0.897\% \\
WEC & 835.6 & 114.6 & 86.23\% & 1.020\% \\
kernel & 667.0 & 118.3 & 82.26\% & 0.000\% \\
\bottomrule
\end{tabular}
\end{adjustbox}
\end{table*}

% #### Added for reviewer response; table text remains black for readability.
\begin{table*}[t]
\centering
\small
\caption{Mean within-split quantiles of final cell size and selected prefix depth across 50 splits. Each entry is q10 / median / q90.}
\label{tab:reviewer_cell_depth_diagnostics}
\begin{adjustbox}{max width=\textwidth}
\begin{tabular}{lcc}
\toprule
Dataset & Final cell size & Selected depth \\
\midrule
airfoil & 69.1 / 88.0 / 117.9 & 1.00 / 1.28 / 1.91 \\
winered & 69.9 / 90.0 / 114.2 & 1.03 / 1.43 / 2.08 \\
star & 68.7 / 89.5 / 125.6 & 1.68 / 3.15 / 5.21 \\
winewhite & 59.4 / 85.8 / 129.5 & 1.00 / 1.64 / 2.08 \\
cycle & 60.7 / 87.5 / 129.8 & 1.00 / 1.12 / 2.12 \\
electric & 59.7 / 87.5 / 128.9 & 1.00 / 1.90 / 2.00 \\
bike & 59.4 / 86.9 / 130.7 & 1.00 / 1.14 / 2.54 \\
meps19 & 57.6 / 91.9 / 160.9 & 2.26 / 8.92 / 18.74 \\
superconductivity & 57.3 / 84.3 / 134.1 & 1.02 / 2.00 / 2.20 \\
WEC & 73.8 / 107.8 / 175.1 & 2.00 / 2.22 / 5.72 \\
kernel & 335.9 / 482.1 / 724.8 & 1.00 / 1.00 / 4.06 \\
\bottomrule
\end{tabular}
\end{adjustbox}
\end{table*}

% #### Added for reviewer response; table text remains black for readability.
\begin{table*}[t]
\centering
\small
\caption{Weighting ablation with fitted tree order fixed. Each entry is unweighted minus weighted LoBoost with a paired 95\% confidence interval; positive $\Delta$SMIS favors weighted merging.}
\label{tab:reviewer_weighting_ablation}
\begin{adjustbox}{max width=\textwidth}
\begin{tabular}{lcccc}
\toprule
Dataset & $\Delta$SMIS [95\% CI] & $\Delta$ width [95\% CI] & $\Delta$ coverage [95\% CI] & $\Delta$ post-fit s [95\% CI] \\
\midrule
WEC & 3119.600 [2667.454, 3571.746] & 630.360 [331.959, 928.760] & -5.41e-04 [-0.0016, 4.81e-04] & -0.037 [-0.079, 0.005] \\
airfoil & 0.015 [-0.062, 0.092] & 0.035 [-0.032, 0.102] & 0.0012 [-0.0023, 0.0046] & 1.13e-04 [-8.35e-04, 0.001] \\
bike & 5.070 [3.850, 6.289] & 1.569 [0.980, 2.157] & -0.0014 [-0.0031, 2.63e-04] & 0.004 [-0.002, 0.011] \\
cycle & 0.011 [-0.022, 0.043] & 5.94e-04 [-0.030, 0.031] & -2.79e-04 [-0.0013, 7.77e-04] & -4.78e-04 [-0.007, 0.006] \\
electric & 1.05e-04 [-2.78e-05, 2.38e-04] & -9.18e-06 [-1.57e-04, 1.39e-04] & -1.87e-04 [-0.0017, 0.0013] & 4.23e-04 [-1.62e-04, 0.001] \\
kernel & -0.005 [-0.033, 0.023] & -0.010 [-0.022, 0.003] & 7.51e-05 [-1.61e-04, 3.11e-04] & 0.071 [-0.374, 0.516] \\
meps19 & 0.632 [0.157, 1.106] & 0.216 [-0.145, 0.577] & -8.61e-04 [-0.0019, 2.27e-04] & -5.11e-04 [-0.006, 0.005] \\
star & 1.278 [-0.850, 3.406] & 1.734 [-1.243, 4.712] & -3.08e-04 [-0.0021, 0.0015] & -1.60e-05 [-7.86e-04, 7.54e-04] \\
superconductivity & 0.727 [0.559, 0.895] & 0.193 [0.058, 0.329] & -6.33e-04 [-0.0018, 4.94e-04] & 0.007 [-0.008, 0.022] \\
winered & 0.002 [-0.008, 0.011] & 0.004 [-0.007, 0.014] & 0.0012 [-4.84e-04, 0.0030] & -5.65e-05 [-6.58e-04, 5.45e-04] \\
winewhite & 0.002 [-0.007, 0.011] & -5.83e-04 [-0.011, 0.010] & -1.36e-04 [-0.0018, 0.0015] & -0.002 [-0.003, 2.17e-05] \\
\bottomrule
\end{tabular}
\end{adjustbox}
\end{table*}

% #### Added for reviewer response; table text remains black for readability.
\begin{table*}[t]
\centering
\small
\caption{Tree-order ablation with weighted Hamming fixed. Each entry is random order minus fitted order with a paired 95\% confidence interval; positive $\Delta$SMIS favors the fitted boosting order.}
\label{tab:reviewer_order_ablation}
\begin{adjustbox}{max width=\textwidth}
\begin{tabular}{lcccc}
\toprule
Dataset & $\Delta$SMIS [95\% CI] & $\Delta$ width [95\% CI] & $\Delta$ coverage [95\% CI] & $\Delta$ post-fit s [95\% CI] \\
\midrule
WEC & 2445.787 [1969.530, 2922.043] & 242.151 [-98.302, 582.604] & -0.0025 [-0.0037, -0.0013] & 9.270 [8.342, 10.197] \\
airfoil & -0.077 [-0.236, 0.082] & 0.035 [-0.079, 0.150] & 0.0019 [-0.0040, 0.0077] & 0.002 [3.45e-04, 0.003] \\
bike & 15.504 [13.329, 17.678] & 5.602 [4.431, 6.773] & -0.0020 [-0.0045, 4.91e-04] & 0.466 [0.387, 0.545] \\
cycle & -0.030 [-0.071, 0.011] & -0.004 [-0.046, 0.038] & -0.0016 [-0.0030, -1.51e-04] & 0.154 [0.114, 0.194] \\
electric & 2.37e-04 [3.49e-05, 4.39e-04] & 1.81e-04 [8.00e-06, 3.55e-04] & 9.20e-04 [-0.0012, 0.0030] & 0.087 [0.073, 0.101] \\
kernel & -0.023 [-0.055, 0.010] & -0.010 [-0.024, 0.004] & 3.04e-05 [-2.61e-04, 3.21e-04] & 19.299 [14.044, 24.555] \\
meps19 & 2.204 [1.317, 3.092] & 0.606 [-0.047, 1.259] & -9.80e-04 [-0.0028, 8.28e-04] & 0.139 [0.089, 0.188] \\
star & -2.220 [-5.715, 1.276] & 0.488 [-3.596, 4.573] & 0.0024 [-0.0011, 0.0059] & 8.03e-04 [-2.92e-04, 0.002] \\
superconductivity & 1.792 [1.479, 2.104] & 0.510 [0.336, 0.684] & -6.14e-04 [-0.0023, 0.0011] & 1.356 [1.055, 1.656] \\
winered & 0.009 [-0.009, 0.028] & 0.031 [0.010, 0.052] & 0.0034 [-9.15e-04, 0.0077] & 0.002 [6.89e-04, 0.003] \\
winewhite & 0.017 [0.006, 0.027] & 0.002 [-0.010, 0.015] & -8.98e-04 [-0.0026, 8.20e-04] & 0.037 [0.026, 0.047] \\
\bottomrule
\end{tabular}
\end{adjustbox}
\end{table*}

% #### Added for reviewer response; table text remains black for readability.
\begin{table*}[t]
\centering
\small
\caption{Mean $\eta B_t$ by relative tree-index quartile over 50 splits. Q1 contains the earliest fitted trees and Q4 the latest.}
\label{tab:reviewer_tree_contributions}
\begin{adjustbox}{max width=\textwidth}
\begin{tabular}{lcccc}
\toprule
Dataset & Q1 $\eta B_t$ & Q2 $\eta B_t$ & Q3 $\eta B_t$ & Q4 $\eta B_t$ \\
\midrule
WEC & 4800.896 & 1273.959 & 904.146 & 730.253 \\
airfoil & 0.403 & 0.133 & 0.091 & 0.074 \\
bike & 12.361 & 2.590 & 1.698 & 1.411 \\
cycle & 0.797 & 0.162 & 0.132 & 0.117 \\
electric & 0.009 & 0.005 & 0.003 & 0.002 \\
kernel & 14.645 & 1.171 & 0.207 & 0.101 \\
meps19 & 2.085 & 1.431 & 1.263 & 1.175 \\
star & 12.827 & 8.198 & 6.951 & 6.173 \\
superconductivity & 1.779 & 0.619 & 0.490 & 0.414 \\
winered & 0.035 & 0.018 & 0.014 & 0.012 \\
winewhite & 0.048 & 0.025 & 0.020 & 0.018 \\
\bottomrule
\end{tabular}
\end{adjustbox}
\end{table*}

\subsection{Scale-only synthetic failure case}
We simulate $X_1,X_2\overset{\mathrm{iid}}{\sim}\mathrm{Uniform}(-2,2)$ and $Y=2\sin(1.5X_1)+0.5X_1+\{0.20+1.20/(1+\exp(-4X_2))\}\varepsilon$, with $\varepsilon\sim\mathcal N(0,1)$. Thus $X_1$ controls the conditional mean and independent $X_2$ controls residual scale. Across 30 matched 60/25/15 splits, the fitted GBDT assigns 99.34\% of gain importance to $X_1$ and only 0.66\% to $X_2$. \ourmethod improves slightly over ICP in SMIS and WSC but is materially worse than WICP and LoCART, making the intended failure mode explicit. Figure~\ref{fig:reviewer_scale_only_coverage} shows that this gap is driven by poor adaptation to the scale covariate: ICP and \ourmethod substantially over-cover at low values of $X_2$ and under-cover at high values, whereas WICP and LoCART remain much closer to the nominal coverage level across scale deciles. The interval-width diagnostics in Figure~\ref{fig:reviewer_scale_only_width} provide the corresponding mechanism. ICP remains essentially constant in $X_2$, while \ourmethod changes its width only modestly and fails to reproduce the strong increase in residual scale. In contrast, WICP and LoCART closely track the oracle width pattern along $X_2$. The width surfaces also show that \ourmethod retains substantial variation along $X_1$, consistent with the fitted mean model being dominated by the mean-related covariate rather than the scale-only covariate.

% #### Added for reviewer response; table text remains black for readability.
\begin{table*}[t]
\centering
\small
\caption{Scale-only synthetic experiment over 30 matched splits. $X_1$ controls the conditional mean and independent $X_2$ controls residual scale. Values are mean $\pm$ standard error.}
\label{tab:reviewer_scale_only}
\begin{adjustbox}{max width=\textwidth}
\begin{tabular}{lcccc}
\toprule
Method & Coverage & SMIS & Mean width & WSC \\
\midrule
ICP & 0.9011 $\pm$ 0.0022 & 4.631 $\pm$ 0.026 & 3.370 $\pm$ 0.025 & 0.7256 $\pm$ 0.0061 \\
LoBoost & 0.9067 $\pm$ 0.0019 & 4.576 $\pm$ 0.027 & 3.419 $\pm$ 0.025 & 0.7494 $\pm$ 0.0056 \\
LoCART & 0.9004 $\pm$ 0.0023 & 3.434 $\pm$ 0.027 & 2.703 $\pm$ 0.020 & 0.8225 $\pm$ 0.0041 \\
Oracle & 0.8978 $\pm$ 0.0017 & 3.295 $\pm$ 0.022 & 2.622 $\pm$ 0.011 & 0.8326 $\pm$ 0.0028 \\
WICP & 0.8975 $\pm$ 0.0021 & 3.392 $\pm$ 0.024 & 2.617 $\pm$ 0.023 & 0.8211 $\pm$ 0.0044 \\
\bottomrule
\end{tabular}
\end{adjustbox}
\end{table*}

\begin{figure*}[t]
  \centering
  \begin{subfigure}[b]{0.48\textwidth}
    \includegraphics[width=\linewidth]{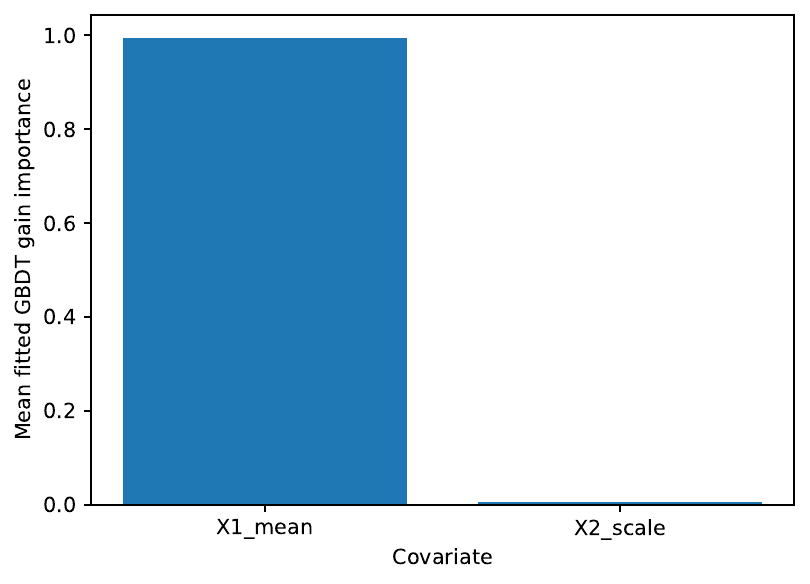}
    \caption{GBDT feature use.}
  \end{subfigure}\hfill
  \begin{subfigure}[b]{0.48\textwidth}
    \includegraphics[width=\linewidth]{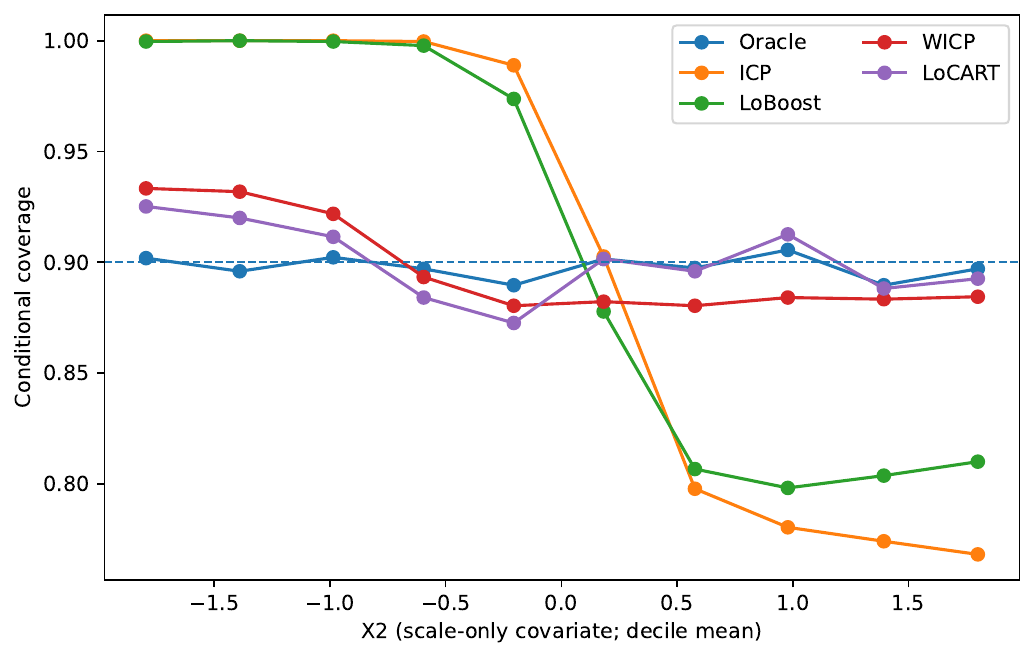}
    \caption{Conditional coverage by scale decile.}
  \end{subfigure}
  \caption{Diagnostics for the scale-only synthetic experiment.}
  \label{fig:reviewer_scale_only_coverage}
\end{figure*}
\begin{figure*}[t]
  \centering
  \begin{subfigure}[b]{0.48\textwidth}
    \includegraphics[width=\linewidth]{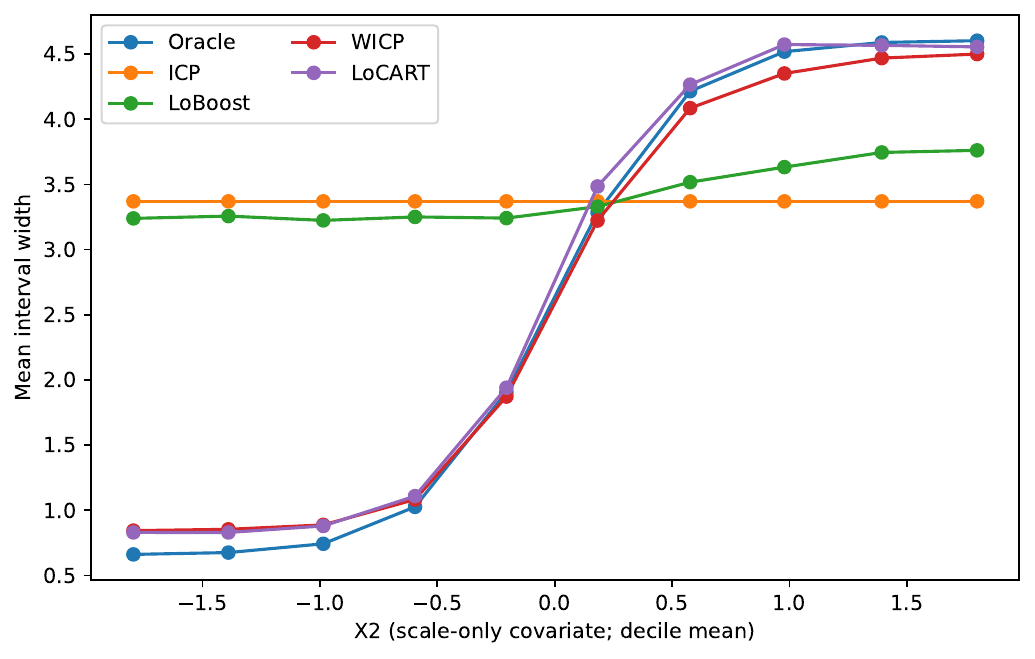}
    \caption{Interval width by scale decile.}
  \end{subfigure}\hfill
  \begin{subfigure}[b]{0.48\textwidth}
    \includegraphics[width=\linewidth]{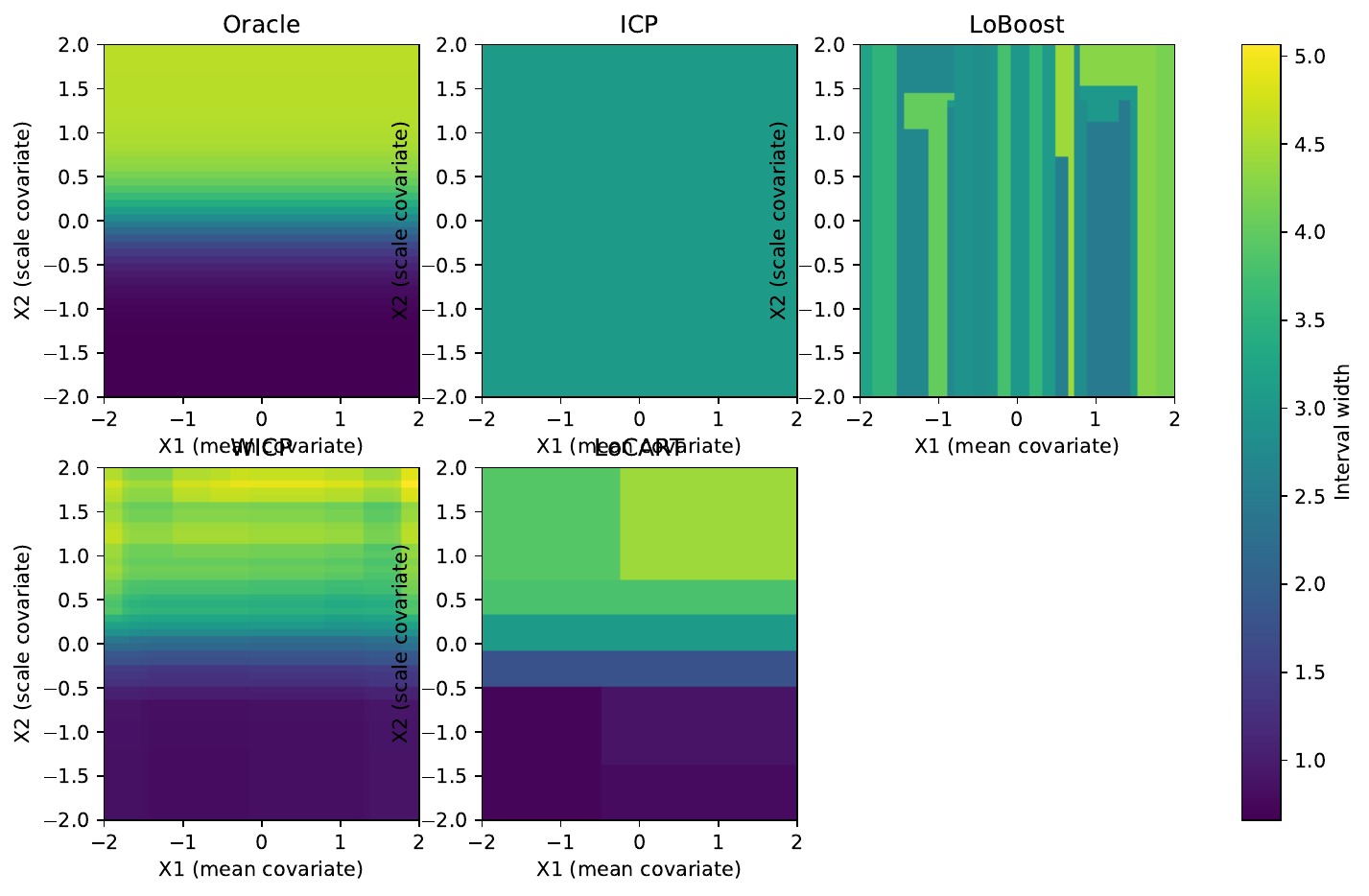}
    \caption{Interval-width surfaces.}
  \end{subfigure}
  \caption{Limited scale adaptation when residual scale is not represented in the fitted mean model.}
  \label{fig:reviewer_scale_only_width}
\end{figure*}

\clearpage

\section{Proofs of Theoretical Results}
\label{appendix:proofs}

The proofs follow the order of the theoretical development in the
main text.  We first establish the structural leaf-value bound, then treat
partitions frozen by the training data, and finally consider partitions
selected from calibration covariates.  All conditioning sigma-fields are
stated explicitly in each proof.

\begin{proof}[Proof of Proposition \ref{proposition:leaf_value_stability}]
Recall that $g_T(u)=g_0+\eta\sum_{t=1}^T h_t(u)$ with constant $g_0$.
The triangle inequality gives
\begin{align*}
|g_T(x)-g_T(z)|
&=
\left|\eta\sum_{t=1}^T\bigl(h_t(x)-h_t(z)\bigr)\right|\\
&\le
\eta\sum_{t=1}^T|h_t(x)-h_t(z)|
=d_{\mathrm{LV}}(x,z).
\end{align*}
The reverse triangle inequality then gives, for every $y\in\R$,
\[
\bigl||y-g_T(x)|-|y-g_T(z)|\bigr|
\le |g_T(x)-g_T(z)|
\le d_{\mathrm{LV}}(x,z).
\]
Finally, if $z\in\cR_k(x)$, then $h_t(z)=h_t(x)$ for every $t\le k$.
Therefore
\[
d_{\mathrm{LV}}(x,z)
=\eta\sum_{t=k+1}^T|h_t(x)-h_t(z)|
\le\eta\sum_{t=k+1}^T B_t
=b_{T,k}.
\]
\end{proof}

\begin{proof}[Proof of Theorem \ref{theorem:local_coverage}]
Fix a cell $A\in\Pi$.  Because $\Pi$ is
$\mathcal F_{\mathrm{tr}}$-measurable, $A$ is fixed after conditioning on the
training data.  Recall that
\[
I_A\coloneqq\{i\le n_{\cal}:X_i\in A\},
\qquad
M_A\coloneqq|I_A|,
\qquad
S_i\coloneqq|Y_i-g_T(X_i)|,
\]
and that
$r_m\coloneqq\lceil(m+1)(1-\alpha)\rceil$.  The cutoff $\widehat q_A$ is the
conformal empirical quantile associated with rank $r_m$, using the convention
defined in the main text.

Conditional on $\mathcal F_{\mathrm{tr}}$, $M_A=m$, and
$X_{n+1}\in A$, the $m$ selected calibration scores and the test score are
exchangeable draws from the score distribution conditional on $X\in A$.
Introduce independent continuous tie-breakers and rank the resulting
score--tie-breaker pairs. Exchangeability makes the test pair's rank uniform
on $\{1,\ldots,m+1\}$. Without tie-breaking, the corresponding conformal
p-value is super-uniform, so ties can only make the interval more conservative.
In particular, when $r_m\le m$, the event that the tie-broken rank is at most $r_m$ is
contained in $\{Y_{n+1}\in\widehat C(X_{n+1})\}$, so
\[
\Pr{Y_{n+1}\in\widehat C(X_{n+1})
\mid\mathcal F_{\mathrm{tr}},M_A=m,X_{n+1}\in A}
\ge \frac{r_m}{m+1}\ge1-\alpha.
\]
If $r_m=m+1$, the cutoff is $+\infty$ and the same lower bound is immediate.
Averaging first over $M_A$ establishes group-conditional validity. Averaging
again over the training-frozen cells yields
\begin{align*}
&\Pr{
Y_{n+1}\in\widehat C(X_{n+1})
\mid\mathcal F_{\mathrm{tr}}
}\\
&\quad=
\sum_{A\in\Pi}
\Pr{X_{n+1}\in A\mid\mathcal F_{\mathrm{tr}}}
\Pr{
Y_{n+1}\in\widehat C(X_{n+1})
\mid\mathcal F_{\mathrm{tr}},X_{n+1}\in A
}
\ge1-\alpha.
\end{align*}
\end{proof}

For a fixed training-frozen cell $A$, write
\[
F_A(t)
\coloneqq
\Pr{|Y-g_T(X)|\le t\mid\mathcal F_{\mathrm{tr}},X\in A}
\]
for its conditional score CDF.

\begin{proposition}[Realized local coverage distribution]
\label{proposition:local_threshold_variability}
Suppose that $F_A$ is continuous, condition on $M_A=m$, and assume
$r_m\le m$.  Then
\[
\widehat p_A
\coloneqq
F_A(\widehat q_A)
\sim
\mathrm{Beta}(r_m,m+1-r_m).
\]
In particular,
\[
\E{\widehat p_A\mid\mathcal F_{\mathrm{tr}},M_A=m}
=\frac{r_m}{m+1},
\qquad
\V{\widehat p_A\mid\mathcal F_{\mathrm{tr}},M_A=m}
=
\frac{r_m(m+1-r_m)}{(m+1)^2(m+2)}.
\]
\end{proposition}

\begin{proof}[Proof of Proposition \ref{proposition:local_threshold_variability}]
Condition on $\mathcal F_{\mathrm{tr}}$ and $M_A=m$, with
$r_m\le m$.
Relabel the $m$ local calibration scores as
$S_1^A,\ldots,S_m^A$ and set $U_i\coloneqq F_A(S_i^A)$.
Continuity of $F_A$ and the probability integral
transform give $U_1,\ldots,U_m\overset{\mathrm{i.i.d.}}{\sim}
\mathrm{Unif}(0,1)$.  Moreover,
\[
F_A(S^A_{(r_m)})=U_{(r_m)}.
\]
The $r_m$-th order statistic of $m$ independent uniform random variables has
the $\mathrm{Beta}(r_m,m+1-r_m)$ distribution.  Its mean and variance are
\[
\frac{r_m}{m+1},
\qquad
\frac{r_m(m+1-r_m)}{(m+1)^2(m+2)},
\]
respectively.
\end{proof}

\begin{proof}[Proof of Theorem \ref{theorem:conditional_coverage}]
For concise notation, recall that
\[
p_n(x)
\coloneqq
\Pr{X\in A_n(x)\mid\mathcal F_{\mathrm{tr},n}},
\qquad
M_n(x)
\coloneqq
\sum_{i=1}^{n_{\cal,n}}
\Ind{X_i\in A_n(x)},
\]
and set
\[
\lambda_n\coloneqq n_{\cal,n}p_n(x).
\]

For use below, write
\begin{align*}
F_{n,A_n(x)}(t)
&\coloneqq
\Pr{|Y-g_n(X)|\le t
\mid\mathcal F_{\mathrm{tr},n},X\in A_n(x)},\\
F_{n,x}(t)
&\coloneqq
\Pr{|Y-g_n(x)|\le t
\mid\mathcal F_{\mathrm{tr},n},X=x}.
\end{align*}
Here $M_n(x)$ is the realized number of calibration observations in $A_n(x)$,
while $\lambda_n$ is its conditional mean.

\emph{Step 1: the local calibration count diverges.}
Because the partition is frozen by the training sample, conditionally on
$\mathcal F_{\mathrm{tr},n}$ the indicators in the definition of $M_n(x)$ are
independent Bernoulli variables with success probability $p_n(x)$.  Hence
\[
M_n(x)\mid\mathcal F_{\mathrm{tr},n}
\sim
\mathrm{Binomial}(n_{\cal,n},p_n(x)),
\qquad
\E{M_n(x)\mid\mathcal F_{\mathrm{tr},n}}=\lambda_n.
\]

We show that $M_n(x)\xrightarrow{p}\infty$.  Fix an integer $K\ge0$ and choose
$L>2K$.  On the event $\{\lambda_n\ge L\}$, we have
$K<\lambda_n/2$.  Consequently,
\[
\{M_n(x)\le K\}
\subseteq
\{M_n(x)\le\lambda_n/2\}
\qquad\text{on }\{\lambda_n\ge L\}.
\]
For a binomial variable with mean $\lambda_n$, the lower-tail Chernoff bound
states that, for every $\delta\in(0,1)$,
\[
\Pr{
M_n(x)\le(1-\delta)\lambda_n
\mid\mathcal F_{\mathrm{tr},n}
}
\le
\exp\left(-\frac{\delta^2\lambda_n}{2}\right).
\]
Taking $\delta=1/2$ and using the event inclusion above gives
\begin{align*}
\Pr{M_n(x)\le K\mid\mathcal F_{\mathrm{tr},n}}
&\le
\Pr{M_n(x)\le\lambda_n/2\mid\mathcal F_{\mathrm{tr},n}}\\
&\le
\exp(-\lambda_n/8)
\le
\exp(-L/8),
\end{align*}
on $\{\lambda_n\ge L\}$.  We now separate two possibilities.  If
$\lambda_n<L$, the conditional mean is not yet guaranteed to be large.  If
$\lambda_n\ge L$, the preceding Chernoff argument bounds the probability of
$M_n(x)\le K$ by $e^{-L/8}$.  Therefore, decomposing according to these two
events and averaging over the training sample yields
\begin{align*}
\Pr{M_n(x)\le K}
&=
\Pr{M_n(x)\le K,\lambda_n<L}
+
\Pr{M_n(x)\le K,\lambda_n\ge L}\\
&\le
\Pr{\lambda_n<L}
+e^{-L/8}\Pr{\lambda_n\ge L}\\
&\le
\Pr{\lambda_n<L}+e^{-L/8}.
\end{align*}
Equivalently, the only two ways in which the realized count can remain below
$K$ are that its conditional mean has not yet exceeded $L$, or that a
lower-tail deviation occurs despite the large conditional mean.

The assumption $\lambda_n=n_{\cal,n}p_n(x)\xrightarrow{p}\infty$ means that,
for every fixed $L$,
\[
\Pr{\lambda_n<L}\longrightarrow0.
\]
Consequently,
\[
\limsup_{n\to\infty}\Pr{M_n(x)\le K}\le e^{-L/8}.
\]
Letting $L\to\infty$ proves that $\Pr{M_n(x)\le K}\to0$ for every fixed $K$,
which is precisely $M_n(x)\xrightarrow{p}\infty$.

Set
\[
m_\alpha
\coloneqq
\left\lceil\frac{1}{\alpha}\right\rceil-1.
\]
If $m\ge m_\alpha$, then $\alpha(m+1)\ge1$ and hence
\[
r_m
=
\left\lceil(m+1)(1-\alpha)\right\rceil
\le m.
\]
Thus $m_\alpha$ is a fixed threshold beyond which $r_m\le m$.
Since Step 1 gives $M_n(x)\xrightarrow{p}\infty$, the event
$\{M_n(x)<m_\alpha\}$ has probability tending to zero.

\emph{Step 2: the realized cell-level coverage converges.}
Let
\[
U_n
\coloneqq
F_{n,A_n(x)}(\widehat q_{A_n(x)}),
\]
where $F_{n,A_n(x)}$ is the conditional score CDF in $A_n(x)$.  Conditional on
$\mathcal F_{\mathrm{tr},n}$ and $M_n(x)=m\ge m_\alpha$,
Proposition~\ref{proposition:local_threshold_variability} gives
\[
U_n\sim\mathrm{Beta}(r_m,m+1-r_m).
\]
Writing $\mu_m\coloneqq r_m/(m+1)$, the ceiling definition of $r_m$ implies
\[
0\le \mu_m-(1-\alpha)<\frac{1}{m+1},
\]
while the Beta variance satisfies
\[
\V{U_n\mid\mathcal F_{\mathrm{tr},n},M_n(x)=m}
=
\frac{r_m(m+1-r_m)}{(m+1)^2(m+2)}
\le
\frac{1}{4(m+2)}.
\]
Fix $\varepsilon>0$ and choose an integer $K\ge m_\alpha$ sufficiently large
that $1/(K+1)<\varepsilon/2$.  For every $m\ge K$, the bias bound and
Chebyshev's inequality give
\begin{align*}
&\Pr{
|U_n-(1-\alpha)|>\varepsilon
\mid\mathcal F_{\mathrm{tr},n},M_n(x)=m
}\\
&\quad\le
\Pr{
|U_n-\mu_m|>\varepsilon/2
\mid\mathcal F_{\mathrm{tr},n},M_n(x)=m
}\\
&\quad\le
\frac{4}{\varepsilon^2}
\V{U_n\mid\mathcal F_{\mathrm{tr},n},M_n(x)=m}
\le
\frac{1}{\varepsilon^2(m+2)}
\le
\frac{1}{\varepsilon^2(K+2)}.
\end{align*}

To make the averaging step explicit, let
\[
\mathcal E_{n,\varepsilon}
\coloneqq
\{|U_n-(1-\alpha)|>\varepsilon\}.
\]
Split this event according to whether the realized local calibration count is
smaller than $K$.  By iterated expectation,
\begin{align*}
\Pr{\mathcal E_{n,\varepsilon}}
&=
\Pr{\mathcal E_{n,\varepsilon},M_n(x)<K}\\
&\quad+
\E{
\Ind{M_n(x)\ge K}
\Pr{\mathcal E_{n,\varepsilon}
\mid\mathcal F_{\mathrm{tr},n},M_n(x)}
}\\
&\le
\Pr{M_n(x)<K}
+
\frac{1}{\varepsilon^2(K+2)}\Pr{M_n(x)\ge K}\\
&\le
\Pr{M_n(x)<K}
+
\frac{1}{\varepsilon^2(K+2)}.
\end{align*}
The first term is the probability that the cell contains too few calibration
points for the uniform Beta bound to be useful.  This is precisely where Step
1 is used: since $M_n(x)\xrightarrow{p}\infty$, for every fixed $K$,
$\Pr{M_n(x)<K}\to0$.

\emph{Step 3: transfer from the cell to the fixed point.}
Because the test response is independent of the calibration sample
conditionally on the training data and $X=x$, the realized pointwise coverage
is
\[
\operatorname{Cov}_{n,x}(\widehat C_n)
\coloneqq
\Pr{Y\in\widehat C_n(x)
\mid X=x,\mathcal F_{\mathrm{tr},n},\cD_{\cal,n}}
=
F_{n,x}(\widehat q_{A_n(x)}).
\]
Adding and subtracting
$F_{n,A_n(x)}(\widehat q_{A_n(x)})=U_n$ gives
\begin{align*}
\left|
\operatorname{Cov}_{n,x}(\widehat C_n)-(1-\alpha)
\right|
&\le
\left|
F_{n,x}(\widehat q_{A_n(x)})
-F_{n,A_n(x)}(\widehat q_{A_n(x)})
\right|+
|U_n-(1-\alpha)|\\
&\le
\Delta_n(x)
+|U_n-(1-\alpha)|.
\end{align*}
The first term converges to zero in probability by assumption, and the second
does so by Step 2.  Their sum therefore converges to zero in probability,
which proves the claimed asymptotic pointwise coverage.
\end{proof}

\subsection{\texorpdfstring{Theoretical Guarantees for Resolution Selection}{Theoretical Guarantees for Resolution Selection}}
\label{sec:calibration_adaptive_partition}

% Proof of Theorem~\ref{theorem:calibration_covariate_adaptive} (stated in Section~\ref{sec:practical}):

\begin{proof}[Proof of Theorem \ref{theorem:calibration_covariate_adaptive}]

We prove two claims.  First, for a finite calibration sample, we must bound
the difference between the realized coverage at $x$ and the target
$1-\alpha$.  The main difficulty is that, after fixing the calibration
covariates, the local scores are independent but need not be identically
distributed.  We therefore need both a bound on their discrepancy from the
target score distribution, provided by $\Omega_n(x)$, and concentration of
their empirical CDF, provided by the local sample size $M_n(x)$.  We condition
on the calibration covariates, use a probability-integral transform and the
Dvoretzky--Kiefer--Wolfowitz inequality to control the empirical quantile,
and then account for the conformal ceiling correction.  Second, we let
$M_n(x)$ diverge and $\Omega_n(x)$ vanish to obtain the asymptotic conclusion.

Condition on $\mathfrak G_n$ and abbreviate
\[
A=A_n(x),
\qquad
M=M_n(x),
\qquad
\omega=\Omega_n(x),
\qquad
F=F_{n,x}.
\]
The selected cell and the index set
$I_A=\{i\le n_{\cal,n}:X_i\in A\}$ are now fixed.  Write
$I_A=\{i_1,\ldots,i_M\}$, set
$S_j^A=|Y_{i_j}-g_n(X_{i_j})|$, and let $G_j$ be its conditional CDF.
Response-blind selection ensures that, conditionally on $\mathfrak G_n$,
these scores are independent and
\[
G_j=F_{n,X_{i_j}},
\qquad
\sup_{t\ge0}|G_j(t)-F(t)|\le\omega.
\]
A fresh score at $x$ is independent of the calibration responses and has CDF
$F$.

Define the empirical CDF of the local scores by
\[
\widehat F_A(t)
\coloneqq
\frac1M\sum_{j=1}^M\Ind{S_j^A\le t},
\]
which we compare with the target CDF $F$.  Define $U_j=G_j(S_j^A)$ and
\[
H_M(u)
\coloneqq
\frac1M\sum_{j=1}^M\Ind{U_j\le u}.
\]
By continuity and the probability integral transform, for every
$u\in[0,1]$,
\[
\Pr{U_j\le u\mid\mathfrak G_n}=u.
\]
Because the scores $S_1^A,\ldots,S_M^A$ are conditionally independent, the
variables $U_1,\ldots,U_M$ are conditionally i.i.d. uniform on $[0,1]$.

For every $j$ and $t\ge0$, the definition of $\omega$ gives
\[
F(t)-\omega
\le G_j(t)\le
F(t)+\omega.
\]
After truncating the endpoints to $[0,1]$, monotonicity of the indicator
function yields
\begin{align*}
\Ind{U_j\le(F(t)-\omega)\vee0}
&\le \Ind{U_j\le G_j(t)}\\
&=\Ind{S_j^A\le t}\\
&\le \Ind{U_j\le(F(t)+\omega)\wedge1}
\end{align*}
almost surely.  Summing these inequalities over $j$ and dividing by $M$
gives
\[
H_M\bigl((F(t)-\omega)\vee0\bigr)
\le
\widehat F_A(t)
\le
H_M\bigl((F(t)+\omega)\wedge1\bigr).
\]

Now suppose that
$\sup_{u\in[0,1]}|H_M(u)-u|\le\varepsilon$.  Applying the lower side of
the last display gives
\begin{align*}
\widehat F_A(t)
&\ge H_M\bigl((F(t)-\omega)\vee0\bigr)\\
&\ge (F(t)-\omega)\vee0-\varepsilon\\
&\ge F(t)-\omega-\varepsilon.
\end{align*}
Similarly, the upper side gives
\begin{align*}
\widehat F_A(t)
&\le H_M\bigl((F(t)+\omega)\wedge1\bigr)\\
&\le (F(t)+\omega)\wedge1+\varepsilon\\
&\le F(t)+\omega+\varepsilon.
\end{align*}
Hence
\[
\sup_{t\ge0}|\widehat F_A(t)-F(t)|
\le\omega+\varepsilon.
\]

The Dvoretzky--Kiefer--Wolfowitz inequality applied to the conditional
uniform sample gives
\[
\Pr{
\sup_{u\in[0,1]}|H_M(u)-u|>\varepsilon
\,\middle|\,\mathfrak G_n
}
\le 2e^{-2M\varepsilon^2}.
\]
Taking
\[
\varepsilon
=
\sqrt{\frac{\log(2/\delta)}{2M}}
\]
makes the right-hand side equal to $\delta$.  Thus, with conditional
probability at least $1-\delta$,
\[
\sup_{t\ge0}|\widehat F_A(t)-F(t)|
\le
\omega+\sqrt{\frac{\log(2/\delta)}{2M}}.
\]

Continuity rules out ties.  Using the standard extended-real convention for
the conformal empirical quantile, the ceiling correction gives
\[
\left|
\widehat F_A(\widehat q_A)-(1-\alpha)
\right|
<\frac{2}{M}.
\]

Once the complete calibration sample is observed, $\widehat q_A$ is fixed.
For an independent test response at $x$,
\begin{align*}
\operatorname{Cov}_{n,x}(\widehat C_n)
&=
\Pr{|Y-g_n(x)|\le\widehat q_A
\mid X=x,\mathcal F_{\mathrm{tr},n},\cD_{\cal,n}}\\
&=F(\widehat q_A).
\end{align*}
Consequently,
\begin{align*}
\left|
\operatorname{Cov}_{n,x}(\widehat C_n)-(1-\alpha)
\right|
&=
\left|F(\widehat q_A)-(1-\alpha)\right|\\
&\le
\left|F(\widehat q_A)-\widehat F_A(\widehat q_A)\right|\\
&\quad+
\left|\widehat F_A(\widehat q_A)-(1-\alpha)\right|.
\end{align*}
Using the two bounds above, with conditional probability at least
$1-\delta$ this is at most
\[
\omega
+\sqrt{\frac{\log(2/\delta)}{2M}}
+\frac{2}{M},
\]
which is the stated finite-calibration bound.

\end{proof}

\subsection{Sufficient Conditions for Boosting-Induced Cells}
\label{appendix:boosting_sufficient_conditions}

We now translate the abstract homogeneity conditions in
Theorems~\ref{theorem:conditional_coverage} and
\ref{theorem:calibration_covariate_adaptive} into conditions on cells induced
by the fitted ensemble.  The common task is to control the uniform discrepancy
$\Omega_n(x)$ through the cell geometry and the variation of the fitted
center.  For a training-frozen cell,
Proposition~\ref{proposition:leaf_tail_score_homogeneity} then gives
$\Delta_n(x)\le\Omega_n(x)$, and the expected local size
$n_{\cal,n}p_n(x)$ controls the calibration count.  For a partition selected
from calibration covariates, Theorem~\ref{theorem:calibration_covariate_adaptive}
uses $\Omega_n(x)$ directly together with the realized count $M_n(x)$.  We
therefore establish the common homogeneity bound first and introduce the
appropriate count condition when applying each theorem.

Fix $x\in\mathcal{S}_X$, where $\mathcal{S}_X$ is the support of $X$, and let
$A_n(x)$ be the boosting-induced cell containing $x$.  Recall that
\[
g_n(u)=g_{0,n}+\eta\sum_{t=1}^{T_n}h_{t,n}(u),
\]
where $g_{0,n}$ is constant.  For $z\in\cX$, define
\begin{align*}
H_{n,z}(y)
&\coloneqq
\Pr{Y\le y\mid\mathcal{F}_{\mathrm{tr},n},X=z},\\
F_{n,z}(t)
&\coloneqq
\Pr{|Y-g_n(z)|\le t
\mid\mathcal{F}_{\mathrm{tr},n},X=z}.
\end{align*}
The cell radius, fitted-center oscillation, and uniform score discrepancy are
\begin{align*}
D_n(x)
&\coloneqq
\sup_{z\in A_n(x)\cap\mathcal{S}_X}\|z-x\|,\\
V_n(x)
&\coloneqq
\sup_{z\in A_n(x)\cap\mathcal{S}_X}|g_n(z)-g_n(x)|,\\
\Omega_n(x)
&\coloneqq
\sup_{z\in A_n(x)\cap\mathcal{S}_X}
\sup_{t\ge0}|F_{n,z}(t)-F_{n,x}(t)|.
\end{align*}

The following response-regularity condition connects covariate and
fitted-center variation to score-distribution variation in both regimes.

\begin{assumption}[Conditional response regularity]
\label{assumption:conditional_response_regularity}
For the fixed point \(x\), there exist constants \(L_x,\bar f_x<\infty\),
uniform in \(n\), such that, for every
\(z\in A_n(x)\cap\mathcal{S}_X\),
\[
\sup_{y\in\R}|H_{n,z}(y)-H_{n,x}(y)|
\le L_x\|z-x\|.
\]
Moreover, the conditional law of \(Y\) given
\((\mathcal{F}_{\mathrm{tr},n},X=z)\) has a density \(f_{n,z}\) satisfying
\(\sup_y f_{n,z}(y)\le\bar f_x\).
\end{assumption}

The next proposition converts response regularity, cell radius, and
within-cell center oscillation into a bound on the score distributions.

\begin{proposition}[Score homogeneity within boosting-induced cells]
\label{proposition:leaf_tail_score_homogeneity}
Under Assumption~\ref{assumption:conditional_response_regularity},
\[
\Omega_n(x)
\le
2L_xD_n(x)+2\bar f_xV_n(x).
\]
If \(A_n(x)\) is \(\mathcal{F}_{\mathrm{tr},n}\)-measurable and has positive
conditional probability, then
\[
\Delta_n(x)\le\Omega_n(x),
\]
where $\Delta_n(x)$ is the cell-average discrepancy defined in
Theorem~\ref{theorem:conditional_coverage}.
\end{proposition}

\begin{proof}
Fix a realization of the selected cell and take
\(z\in A_n(x)\cap\mathcal{S}_X\) and \(t\ge0\).  Introduce
\[
\widetilde F_{n,x;z}(t)
\coloneqq
\Pr{|Y-g_n(z)|\le t
\mid\mathcal{F}_{\mathrm{tr},n},X=x}.
\]
Because the conditional response distributions have densities, their CDFs
are continuous.  Therefore
\begin{align*}
F_{n,z}(t)
&=H_{n,z}(g_n(z)+t)-H_{n,z}(g_n(z)-t),\\
\widetilde F_{n,x;z}(t)
&=H_{n,x}(g_n(z)+t)-H_{n,x}(g_n(z)-t).
\end{align*}
Subtracting these two identities and applying the triangle inequality gives
\begin{align*}
&|F_{n,z}(t)-\widetilde F_{n,x;z}(t)|\\
&\quad=
\Big|
H_{n,z}(g_n(z)+t)-H_{n,x}(g_n(z)+t)\\
&\qquad\quad-
H_{n,z}(g_n(z)-t)+H_{n,x}(g_n(z)-t)
\Big|\\
&\quad\le
|H_{n,z}(g_n(z)+t)-H_{n,x}(g_n(z)+t)|\\
&\qquad+
|H_{n,z}(g_n(z)-t)-H_{n,x}(g_n(z)-t)|\\
&\quad\le 2L_x\|z-x\|,
\end{align*}
where the last inequality follows from conditional response regularity at
the two interval endpoints.

Next, both $\widetilde F_{n,x;z}$ and $F_{n,x}$ use the response law at $x$,
but their intervals are centered at $g_n(z)$ and $g_n(x)$, respectively.  We
have
\begin{align*}
F_{n,x}(t)
&=H_{n,x}(g_n(x)+t)-H_{n,x}(g_n(x)-t).
\end{align*}
Since the density of the response at $x$ is bounded by $\bar f_x$, for every
$a,b\in\R$,
\[
|H_{n,x}(a)-H_{n,x}(b)|
\le \bar f_x|a-b|.
\]
Applying this bound at the upper and lower endpoints yields
\begin{align*}
&|\widetilde F_{n,x;z}(t)-F_{n,x}(t)|\\
&\quad\le
|H_{n,x}(g_n(z)+t)-H_{n,x}(g_n(x)+t)|\\
&\qquad+
|H_{n,x}(g_n(z)-t)-H_{n,x}(g_n(x)-t)|\\
&\quad\le
2\bar f_x|g_n(z)-g_n(x)|.
\end{align*}

Finally, another application of the triangle inequality gives
\begin{align*}
|F_{n,z}(t)-F_{n,x}(t)|
&\le
|F_{n,z}(t)-\widetilde F_{n,x;z}(t)|\\
&\quad+
|\widetilde F_{n,x;z}(t)-F_{n,x}(t)|\\
&\le
2L_xD_n(x)
+2\bar f_x|g_n(z)-g_n(x)|\\
&\le
2L_xD_n(x)+2\bar f_xV_n(x).
\end{align*}
Taking the supremum over $z\in A_n(x)\cap\mathcal S_X$ and $t\ge0$ proves
the bound on $\Omega_n(x)$.

If the cell is training-frozen, its conditional score CDF satisfies
\[
F_{n,A_n(x)}(t)
=
\E{F_{n,X}(t)
\mid\mathcal{F}_{\mathrm{tr},n},X\in A_n(x)}.
\]
Consequently,
\begin{align*}
&|F_{n,A_n(x)}(t)-F_{n,x}(t)|\\
&\quad=
\left|
\E{F_{n,X}(t)-F_{n,x}(t)
\mid\mathcal{F}_{\mathrm{tr},n},X\in A_n(x)}
\right|\\
&\quad\le
\E{|F_{n,X}(t)-F_{n,x}(t)|
\mid\mathcal{F}_{\mathrm{tr},n},X\in A_n(x)}\\
&\quad\le\Omega_n(x).
\end{align*}
Taking the supremum over \(t\) proves \(\Delta_n(x)\le\Omega_n(x)\).
\end{proof}

Proposition~\ref{proposition:leaf_tail_score_homogeneity} isolates the two
approximation errors that must vanish.  The term $D_n(x)$ measures geometric
localization of the selected cell, whereas $V_n(x)$ measures the variation of
the fitted center inside that cell.  We now give boosting-specific conditions
for controlling these two terms.

\begin{assumption}[Shrinking selected cell]
\label{assumption:shrinking_cell_diameter}
For the fixed point \(x\),
\[
D_n(x)\xrightarrow{p}0.
\]
\end{assumption}

This assumption is natural for tree-induced cells under the usual geometric
localization regime.  To see the connection, let $\Lambda_{t,n}(x)$ denote the leaf region
of tree $t$ containing $x$.  A prefix cell can be written as
\[
\cR_{n,k_n(x)}(x)
=
\bigcap_{t=1}^{k_n(x)}\Lambda_{t,n}(x).
\]
Therefore, for every $t\le k_n(x)$,
$\cR_{n,k_n(x)}(x)\subseteq \Lambda_{t,n}(x)$, and hence
\begin{align*}
D_n(x)
&=
\sup_{z\in\cR_{n,k_n(x)}(x)\cap\mathcal S_X}\|z-x\|\\
&\le
\min_{1\le t\le k_n(x)}
\operatorname{diam}\bigl(\Lambda_{t,n}(x)\cap\mathcal S_X\bigr).
\end{align*}
Thus shrinking leaves among the first $k_n(x)$ trees are sufficient for
$D_n(x)\to0$.  More generally, even if no individual leaf shrinks in every
coordinate, the intersection may shrink when the ensemble progressively
separates points around $x$.  Increasing the prefix depth cannot increase
$D_n(x)$, although geometric shrinkage is an asymptotic condition on the
tree-growing scheme rather than an automatic property of every fitted
ensemble.

We next control the fitted-center variation.  For a prefix cell
$A_n(x)=\cR_{n,k_n(x)}(x)$, define
\[
B_{t,n}
\coloneqq
\sup_{u,z\in\cX}|h_{t,n}(u)-h_{t,n}(z)|,
\qquad
b_n(x)
\coloneqq
\eta\sum_{t=k_n(x)+1}^{T_n}B_{t,n}.
\]
Thus, $b_n(x)$ is the sequence-level counterpart of $b_{T,k}$,
with $T=T_n$, $k=k_n(x)$, and $B_t=B_{t,n}$.

\begin{assumption}[Vanishing boosting tail]
\label{assumption:vanishing_boosting_tail}
For a prefix cell and the fixed point $x$,
\[
b_n(x)\xrightarrow{p}0.
\]
\end{assumption}

For every $z\in A_n(x)$, the first $k_n(x)$ leaf identifiers agree, so
$h_{t,n}(z)=h_{t,n}(x)$ for $t\le k_n(x)$.  Consequently,
\begin{align*}
|g_n(z)-g_n(x)|
&=
\left|
\eta\sum_{t=1}^{T_n}
\bigl(h_{t,n}(z)-h_{t,n}(x)\bigr)
\right|\\
&=
\left|
\eta\sum_{t=k_n(x)+1}^{T_n}
\bigl(h_{t,n}(z)-h_{t,n}(x)\bigr)
\right|\\
&\le
\eta\sum_{t=k_n(x)+1}^{T_n}
|h_{t,n}(z)-h_{t,n}(x)|\\
&\le
\eta\sum_{t=k_n(x)+1}^{T_n}B_{t,n}\\
&=b_n(x).
\end{align*}
Taking the supremum over $z\in A_n(x)\cap\mathcal S_X$ gives
\[
V_n(x)\le b_n(x).
\]
The tail condition formalizes the idea that the trees after the selected
prefix contribute progressively less variation.  For each $n$, it holds
trivially with $b_n(x)=0$ when $k_n(x)=T_n$.  When $k_n(x)<T_n$, it follows,
for example, if
\[
\eta\bigl(T_n-k_n(x)\bigr)
\max_{k_n(x)<t\le T_n}B_{t,n}
\xrightarrow{p}0,
\]
because the left-hand side bounds $b_n(x)$.  This is compatible with boosting
regimes in which later trees make increasingly small residual corrections,
but it is stated explicitly because such decay is not guaranteed for every
fitting scheme.

For a cell produced by merging, the analogous control is provided directly
by the leaf-value distance.  Define
\[
R_n^{\mathrm{LV}}(x)
\coloneqq
\sup_{z\in A_n(x)\cap\mathcal{S}_X}
d_{\mathrm{LV},n}(x,z),
\qquad
d_{\mathrm{LV},n}(x,z)
\coloneqq
\eta\sum_{t=1}^{T_n}|h_{t,n}(x)-h_{t,n}(z)|.
\]

\begin{assumption}[Vanishing leaf-value radius after merging]
\label{assumption:vanishing_leaf_value_radius}
For a merged cell and the fixed point $x$,
\[
R_n^{\mathrm{LV}}(x)\xrightarrow{p}0.
\]
\end{assumption}

For every $z$ in the merged cell, the triangle inequality gives
\begin{align*}
|g_n(z)-g_n(x)|
&=
\left|
\eta\sum_{t=1}^{T_n}
\bigl(h_{t,n}(z)-h_{t,n}(x)\bigr)
\right|\\
&\le
\eta\sum_{t=1}^{T_n}|h_{t,n}(z)-h_{t,n}(x)|\\
&=d_{\mathrm{LV},n}(x,z)\\
&\le R_n^{\mathrm{LV}}(x).
\end{align*}
Taking the supremum over $z\in A_n(x)\cap\mathcal S_X$ gives
\[
V_n(x)\le R_n^{\mathrm{LV}}(x).
\]
Hence a merging rule whose admissible leaf-value radius vanishes preserves
within-cell center stability.  This condition must still be accompanied by
geometric shrinkage, since similarity of fitted values alone does not imply
that the covariates in the merged cell are close to $x$.

\begin{proposition}[Boosting conditions for score homogeneity]
\label{proposition:boosting_score_homogeneity}
Suppose that
Assumptions~\ref{assumption:conditional_response_regularity} and
\ref{assumption:shrinking_cell_diameter} hold.  Suppose in addition that
either
\begin{enumerate}
    \item[(i)] $A_n(x)$ is a prefix cell and
    Assumption~\ref{assumption:vanishing_boosting_tail} holds; or
    \item[(ii)] $A_n(x)$ is a merged cell and
    Assumption~\ref{assumption:vanishing_leaf_value_radius} holds.
\end{enumerate}
Then
\[
\Omega_n(x)\xrightarrow{p}0.
\]
If, in addition, $A_n(x)$ is $\mathcal F_{\mathrm{tr},n}$-measurable and has
positive conditional probability, then
\[
\Delta_n(x)\xrightarrow{p}0.
\]
\end{proposition}

\begin{proof}
In case (i), the prefix-tail calculation above gives
\[
V_n(x)\le b_n(x)\xrightarrow{p}0.
\]
In case (ii), the leaf-value-radius calculation gives
\[
V_n(x)\le R_n^{\mathrm{LV}}(x)\xrightarrow{p}0.
\]
Thus $V_n(x)\xrightarrow{p}0$ in either case.  By
Proposition~\ref{proposition:leaf_tail_score_homogeneity},
\[
0\le\Omega_n(x)
\le2L_xD_n(x)+2\bar f_xV_n(x).
\]
Fix $\varepsilon>0$ and set
\[
\delta
\coloneqq
\frac{\varepsilon}{2(1+L_x+\bar f_x)}.
\]
If $D_n(x)\le\delta$ and $V_n(x)\le\delta$, then
\begin{align*}
2L_xD_n(x)+2\bar f_xV_n(x)
&\le2(L_x+\bar f_x)\delta\\
&=
\varepsilon\frac{L_x+\bar f_x}{1+L_x+\bar f_x}\\
&<\varepsilon.
\end{align*}
It follows that
\begin{align*}
\Pr{\Omega_n(x)>\varepsilon}
&\le
\Pr{2L_xD_n(x)+2\bar f_xV_n(x)>\varepsilon}\\
&\le
\Pr{D_n(x)>\delta}+\Pr{V_n(x)>\delta}\\
&\longrightarrow0.
\end{align*}
Therefore $\Omega_n(x)\xrightarrow{p}0$.  If the cell is training-frozen,
Proposition~\ref{proposition:leaf_tail_score_homogeneity} also gives
$0\le\Delta_n(x)\le\Omega_n(x)$, and hence
$\Delta_n(x)\xrightarrow{p}0$.
\end{proof}

For prefix cells, increasing $k_n(x)$ simultaneously makes the cell radius
$D_n(x)$ and the tail bound $b_n(x)$ no larger, but it also makes the cell
probability and its calibration count no larger.  Merging can restore local
sample size, at the cost of potentially enlarging the geometric and
leaf-value radii.  Before turning to calibration-covariate selection, we
record the training-frozen case as a direct consequence of the preceding
proposition.

\begin{corollary}[Training-frozen boosting cells]
\label{corollary:leaf_value_conditional_coverage}
Fix $x\in\mathcal{S}_X$ and let $A_n(x)$ be a
$\mathcal{F}_{\mathrm{tr},n}$-measurable boosting-induced cell with positive
conditional probability.  Define
\[
p_n(x)
\coloneqq
\Pr{X\in A_n(x)\mid\mathcal{F}_{\mathrm{tr},n}}.
\]
Under the sampling conditions of
Theorem~\ref{theorem:conditional_coverage}, suppose that the conditions of
Proposition~\ref{proposition:boosting_score_homogeneity} hold and that
\[
n_{\cal,n}p_n(x)\xrightarrow{p}\infty.
\]
Then
\[
\Pr{Y\in\widehat C_n(x)
\mid X=x,\mathcal{F}_{\mathrm{tr},n},\cD_{\cal,n}}
\xrightarrow{p}1-\alpha.
\]
\end{corollary}

\begin{proof}
Proposition~\ref{proposition:boosting_score_homogeneity} gives
$\Delta_n(x)\xrightarrow{p}0$, while the displayed condition gives the
required divergence of the expected local calibration size.

The density condition in
Assumption~\ref{assumption:conditional_response_regularity} also implies, for
$s,t\ge0$ and every $z\in A_n(x)\cap\mathcal S_X$,
\[
|F_{n,z}(t)-F_{n,z}(s)|
\le2\bar f_x|t-s|.
\]
Averaging this inequality conditionally on $X\in A_n(x)$ gives
\[
|F_{n,A_n(x)}(t)-F_{n,A_n(x)}(s)|
\le2\bar f_x|t-s|,
\]
so the cell-level score CDF is continuous.  All conditions of
Theorem~\ref{theorem:conditional_coverage} now hold, which proves the result.
\end{proof}

\subsubsection{Partitions Selected from Calibration Covariates}

Here we need to verify the two conditions of
Theorem~\ref{theorem:calibration_covariate_adaptive}: the uniform discrepancy
$\Omega_n(x)$ must vanish and the realized local calibration size
\[
M_n(x)
\coloneqq
\sum_{i=1}^{n_{\cal,n}}\Ind{X_i\in A_n(x)}
\]
must diverge.  Because the selected cell depends on the calibration
covariates, its realized count is controlled directly rather than through a
binomial argument.

\begin{assumption}[Diverging local calibration size]
\label{assumption:diverging_local_calibration_size}
There exists a deterministic sequence \(m_n\to\infty\) such that
\[
\Pr{M_n(x)\ge m_n}\longrightarrow1.
\]
The additional scaling \(m_n/n_{\cal,n}\to0\) is not required for coverage,
but is a natural compatibility condition if the selected cells are also to
shrink.
\end{assumption}

If the algorithm enforces at least $N_{\min,n}$ observations in every final
cell, this assumption holds with $m_n=N_{\min,n}$ whenever
$N_{\min,n}\to\infty$.

\begin{corollary}[Calibration-covariate-adaptive boosting cells]
\label{corollary:adaptive_leaf_value_conditional_coverage}
Fix $x\in\mathcal{S}_X$ and let $A_n(x)$ be the cell of a
$\mathfrak G_n$-measurable, response-blind boosting-induced partition. Under
the sampling conditions of
Theorem~\ref{theorem:calibration_covariate_adaptive}, suppose that
Assumptions~\ref{assumption:conditional_response_regularity},
\ref{assumption:shrinking_cell_diameter},
and \ref{assumption:diverging_local_calibration_size} hold.  Suppose also
that either $A_n(x)$ is a prefix cell satisfying
Assumption~\ref{assumption:vanishing_boosting_tail}, or it is a merged cell
satisfying Assumption~\ref{assumption:vanishing_leaf_value_radius}.  Then
\[
\operatorname{Cov}_{n,x}(\widehat C_n)
\xrightarrow{p}1-\alpha.
\]
\end{corollary}

\begin{proof}
Fix $K>0$.  Since $m_n\to\infty$, we have $m_n>K$ for all sufficiently large
$n$.  For such $n$,
\[
\Pr{M_n(x)\le K}
\le
\Pr{M_n(x)<m_n}
\longrightarrow0
\]
by Assumption~\ref{assumption:diverging_local_calibration_size}.  Since this
holds for every fixed $K$, it follows that
$M_n(x)\xrightarrow{p}\infty$.

The remaining assumptions are the conditions of
Proposition~\ref{proposition:boosting_score_homogeneity}, which gives
\[
\Omega_n(x)\xrightarrow{p}0.
\]

The density condition in
Assumption~\ref{assumption:conditional_response_regularity} makes the
pointwise score CDFs continuous because, for every
$z\in A_n(x)\cap\mathcal S_X$ and $s,t\ge0$,
\[
|F_{n,z}(t)-F_{n,z}(s)|
\le2\bar f_x|t-s|.
\]
We have therefore verified $M_n(x)\xrightarrow{p}\infty$,
$\Omega_n(x)\xrightarrow{p}0$, and continuity.  The conclusion follows from
Theorem~\ref{theorem:calibration_covariate_adaptive}.
\end{proof}

\end{document}